\documentclass[journal]{IEEEtran}

\usepackage{amsmath,amsfonts,amssymb}
\usepackage{amsthm}
\usepackage{array}
\usepackage{booktabs}
\usepackage[caption=false,font=normalsize,labelfont=sf,textfont=sf]{subfig}
\usepackage{textcomp}
\usepackage{graphicx}
\usepackage{url}
\usepackage{verbatim}
\usepackage{stfloats}

\usepackage{algorithm}
\usepackage{algpseudocode}

\usepackage{tikz}
\usetikzlibrary{arrows.meta}

\usepackage{balance}
\usepackage[hidelinks]{hyperref}
\newtheorem{theorem}{Theorem}[section]
\newtheorem{definition}[theorem]{Definition}
\newtheorem{corollary}[theorem]{Corollary}

\newtheorem{remark}[theorem]{Remark}
\newtheorem{conjecture}[theorem]{Conjecture}

\newcommand{\etal}{\textit{et al. }}
\newcommand{\ie}{\textit{i.e. }}
\newcommand{\eg}{\textit{e.g. }}

\def\BibTeX{{\rm B\kern-.05em
  {\sc i\kern-.025em b}\kern-.08em
  T\kern-.1667em\lower.7ex\hbox{E}\kern-.125emX}}

\begin{document}

\title{Amnesia by Design, Memory By Necessity: Persistent State for Document Intelligence}

\author{Souhail~Bakkali~and~Ayoub~Merimi
\thanks{Souhail Bakkali is with the University of Rennes, CNRS, IRISA --- UMR 6074 --- 263 Av. Général Leclerc, Rennes, France.
E-mail: souhail.bakkali@irisa.fr.}
\thanks{Ayoub Merimi is with ENSA Berrechid, Hassan First University,
Settat, Morocco. E-mail: ayoub.merimi@uhp.ma.}
}

\maketitle

\begin{abstract}
Modern Document AI reads contracts, extracts fields, reasons over
tables, and grounds answers to page regions, then forgets
everything. Processing an amendment the next day begins from
scratch: no schema retained, no contradiction detected, no
experience carried forward. This is a structural choice, not a
scale failure: current systems are stateless functions. We call
this the \emph{statelessness bottleneck}. This bottleneck lies
beyond parameter scaling, context extension, and retrieval
augmentation: storage provides persistence and retrieval provides
access, but neither consolidates observations into knowledge that
improves future processing. This survey formalizes \emph{persistent
evidence-grounded document state} as a unifying framework,
specifying the operations and invariants required to convert
multimodal evidence into durable, provenance-linked state. We
introduce a \emph{statefulness audit} showing that ten
representative benchmarks, coded against eight statefulness
criteria, leave cross-session state evolution untested, and derive
a \emph{longitudinal benchmark harness} with five counterfactual
metrics: Experience Gain, Cost Efficiency, Memory Harm, Forgetting
Fidelity, Coverage Retention, to characterize the benefit, cost,
risk, and governability of persistent document state. Document AI
lacks mechanisms coupling persistent state to document-native
structure, provenance, and temporal validity. The next era of
Document AI will be defined by what systems retain across
documents, sessions, and time.
\end{abstract}


\section{Introduction}
\label{sec:intro}

\IEEEPARstart{T}{he} most capable document systems ever built share a property
with the simplest ones: they have no memory. A multimodal foundation model can parse a 200-page financial filing, extract every line item, cross-reference footnotes, and answer a question about a specific table cell on page 147, grounding
its response to the exact bounding box. On established single-page tasks these systems achieve remarkable accuracy, though substantial gaps remain under diverse layouts, long-context reasoning, and unanswerable settings, where even the strongest multimodal foundation models suffer severe performance
degradation~\cite{Deng2025LongDocURL,Masry2025ChartQAPro,wang2024needle}. But the architecture that enables this fluency is the same architecture that prevents it from lasting. The model reads, answers, and resets. The next filing arrives,
and the work begins again from nothing.

This reset is costly because documents carry meaning through their relationship to prior documents. A tax auditor recognizes invoice templates after hundreds of examples; a clinician checks a new lab report against a years-long trajectory; a researcher maintains a running model of what is established, what is contested, and what has been superseded. In each case, the expert's competence is not a property of any single reading. It is a property of accumulated readings. Current Document AI systems possess the first kind of competence in abundance. The second kind they do not possess at all. Training on prior documents encodes statistical regularities into fixed weights that cannot be inspected, localized, selectively revised, or erased without retraining. Competence through accumulated readings requires state that lives beyond the training run.
\begin{figure}[t]
\centering
\begin{tikzpicture}[
  font=\small,
  doc/.style={rectangle, draw, rounded corners=2pt, minimum width=1.1cm, minimum height=0.6cm, fill=blue!8},
  proc/.style={rectangle, draw, rounded corners=2pt, minimum width=2.0cm, minimum height=0.6cm, fill=gray!15},
  outp/.style={rectangle, draw, rounded corners=2pt, minimum width=1.3cm, minimum height=0.6cm, fill=green!8},
  outpw/.style={rectangle, draw, rounded corners=2pt, minimum width=1.9cm, minimum height=0.6cm, fill=green!8},
  rst/.style={rectangle, draw, rounded corners=2pt, minimum width=1.2cm, minimum height=0.6cm, fill=red!12},
  st/.style={rectangle, draw, rounded corners=2pt, minimum width=1.2cm, minimum height=0.6cm, fill=orange!18},
  arr/.style={-{Stealth[length=2mm]}, thick},
  darr/.style={-{Stealth[length=2mm]}, thick, dashed, red!70!black},
  sarr/.style={-{Stealth[length=2mm]}, thick, orange!60!black}
]

\node[font=\small\bfseries, anchor=west] at (-0.6,0.85) {(a) Current: episodic, stateless};
\node[doc] (d1) at (0,0) {$D_1$};
\node[proc, minimum width=1.9cm] (f1) at (2.1,0) {$P(y\mid D_1)$};
\node[outp] (y1a) at (4.2,0) {$y_1$};
\node[rst] (r1) at (5.8,0) {RESET};
\draw[arr] (d1)--(f1); \draw[arr] (f1)--(y1a); \draw[darr] (y1a)--(r1);

\node[doc] (d2) at (0,-1.0) {$D_2$};
\node[proc, minimum width=1.9cm] (f2) at (2.1,-1.0) {$P(y\mid D_2)$};
\node[outp] (y2a) at (4.2,-1.0) {$y_2$};
\node[rst] (r2) at (5.8,-1.0) {RESET};
\draw[arr] (d2)--(f2); \draw[arr] (f2)--(y2a); \draw[darr] (y2a)--(r2);

\node[font=\small\bfseries, anchor=west] at (-0.6,-2.05) {(b) Proposed: longitudinal, stateful};

\node[doc] (d1b) at (0,-3.0) {$D_1$};
\node[proc] (p1) at (2.2,-3.0) {Incorporate};
\node[st]  (s1)  at (4.3,-3.0) {$S_1$};
\node[outp] (y1b) at (5.9,-3.0) {$y_1$};
\draw[arr] (d1b)--(p1); \draw[arr] (p1)--(s1); \draw[arr] (s1)--(y1b);

\node[doc]  (d2b) at (0,-4.5) {$D_2$};
\node[proc] (p2)  at (2.2,-4.5) {Incorporate};
\node[st]   (s2)  at (4.3,-4.5) {$S_2$};
\node[outpw] (y2b) at (6.0,-4.5) {$y_2$ informed};
\draw[arr] (d2b)--(p2); \draw[arr] (p2)--(s2); \draw[arr] (s2)--(y2b);

\draw[sarr] (s1.south) -- ++(0,-0.45) -| (p2.north)
  node[pos=0.25, above, font=\tiny] {persist};

\end{tikzpicture}
\caption{The \emph{statelessness bottleneck}. (a) Current systems
process each document as an isolated function and reset after every
invocation. (b) The proposed paradigm \emph{incorporates} each
document into persistent state at every session, producing an output
at each step: $D_1$ yields $y_1$ and state $S_1$; $S_1$ is carried
forward and merged with the validated evidence of $D_2$ to produce
$S_2$, from which the informed answer $y_2$ is generated. Processing
$D_1$ therefore changes how $D_2$ is understood.}
\label{fig:paradigm}
\end{figure}
The reason is architectural, not accidental. Every mainstream Document AI pipeline, from layout-aware pretraining~\cite{xu2020layoutlm,huang2022layoutlmv3}; to OCR-free generation~\cite{kim2022ocr}; to visual retrieval~\cite{Faysse2025ColPali,Yu2025VisRAG}; to adaptive evidence agents~\cite{asai2024self}, is designed, trained, and evaluated as a stateless function:
\begin{equation}
    \hat{y} = \arg\max_{y}\; P(y \mid D;\, \theta)
    \label{eq:stateless}
\end{equation}
where $D$ is a single document input, $\theta$ denotes frozen parameters, and $\hat{y}$ is the produced output. No variable in Equation~\eqref{eq:stateless} carries a session index. The parameters $\theta$ are fixed at inference and do not accumulate evidence across inputs. The intermediate representations, attention maps, and reasoning traces computed during the forward pass are discarded upon return of $\hat{y}$. No schema is extracted from the encounter. No record states that $D$ was seen, what was learned from it, or how it relates to prior inputs. The
system that produces $\hat{y}$ for the first invoice and the system that produces $\hat{y}$ for the ten-thousandth are, in every computational respect described by
Equation~\eqref{eq:stateless}, identical. We name this pattern the \emph{statelessness bottleneck}, and we argue in this survey that it is the central unresolved problem in Document AI. Figure~\ref{fig:paradigm} contrasts this stateless pattern with the proposed \emph{longitudinal paradigm}, in which validated evidence accumulates into a \emph{persistent state} that conditions future inference.

The bottleneck is not hidden. It is visible in the three engineering responses — \emph{parameter scaling}, \emph{context extension}, and \emph{retrieval augmentation} — each addresses a symptom while leaving the condition untouched. Scaling model parameters incorporates domain knowledge but tangles it into weights that cannot be inspected, localized, or selectively erased. Extending context windows to millions of tokens~\cite{dao2022flashattention,gu2023mamba} provides a workspace wiped when the session ends, and utilization degrades with length~\cite{liu2024lost}. Retrieval-Augmented Generation (RAG) over vector stores~\cite{lewis2020retrieval,karpukhin2020dense} and graph indices~\cite{edge2024local} makes stored content accessible, but \emph{accessibility is not understanding, and lookup is not learning}: an index is reshaped by what is added, not by how its content was reasoned about across queries~\cite{edge2024local,guo2025lightrag}. \emph{Storage means information persists; retrieval means information can be found; consolidation means that prior processing has altered the system's future behavior}. Current Document AI achieves the first two; the third is absent.

The absence is not a gap that scaling will close. Biological memory systems solved this problem through architectural separation, not capacity expansion: Complementary Learning Systems theory~\cite{mcclelland1995there} describes a fast episodic buffer alongside a slow semantic store with offline consolidation~\cite{tse2007schemas}, active forgetting~\cite{teki2017persistence}, and temporal binding~\cite{howard2002distributed}. Machine learning has addressed individual facets: continual learning constrains sequential updates~\cite{kirkpatrick2017overcoming}; agent-memory systems maintain cross-session state~\cite{park2023generative,packer2023memgpt}; associative graph memory supports experience-linked access~\cite{gutierrez2024hipporag}, but none provides the \emph{document-native conjunction} formalized in this survey: no existing mechanism couples persistence to 2D spatial layout, span-level provenance, bitemporal validity, and governed deletion. The consequence is not theoretical. In legal discovery, review teams re-derive findings across sessions because systems do not carry them forward. In clinical documentation, contradictions between successive reports go undetected without a longitudinal model. In financial auditing, provenance chains break between retrieval and reasoning when systems cannot track how conclusions derive from sources, transformations, and time. \emph{The operational need is not a system that reads better. It is a system that remembers what it read, knows where the memory came from, knows when it is still valid, and can be told to forget}.

This survey provides a unified vocabulary, formalism, and evaluation protocol for persistent Document AI. It synthesizes the literature under a transparent coding procedure that distinguishes demonstrated capabilities from inferred and hypothetical ones. Under our operational definition of persistent evidence-grounded state, we identify the \emph{statelessness bottleneck} as a structural property of the task formulations, benchmarks, and training objectives that organize the field, rather than a limitation of any particular architecture. We formalize \emph{persistent evidence-grounded document state} as a six-component representational substrate and specify the lifecycle operations required to maintain it across documents, sessions, and time. We formalize the failure modes that persistence introduces: \textbf{\textit{Memory Harm}}, the risk that consolidated errors contaminate future reasoning; \textbf{\textit{Forgetting Fidelity}}, the requirement that deleted information remain genuinely unrecoverable; and \textbf{\textit{Coverage Retention}}, the requirement that deletion preserve unrelated knowledge. We derive a contamination-accumulation theorem showing that unvalidated writes drive state corruption toward certainty over $k$ sessions, and that a validation gate bounds the accumulation rate. We then audit ten representative benchmarks against eight statefulness criteria, coded under the procedure of Section~\ref{sec:audit_existing}, and show that existing suites evaluate within-episode performance while leaving cross-session state evolution largely untested. From this gap, we derive a \emph{longitudinal benchmark harness} based on five counterfactual metrics: \emph{Experience Gain}, \emph{Cost Efficiency}, \emph{Memory Harm}, \emph{Forgetting Fidelity}, and \emph{Coverage Retention}, to measure the benefit, cost, risk, and governability of persistent document state.

The remainder of this survey is organized as follows.
Section~\ref{sec:evolution} traces the evolution from single-page perception to multimodal retrieval and identifies the boundary between passive access and active consolidation.
Section~\ref{sec:formalization} defines the persistent state tuple, lifecycle operations, and candidate architectural substrates. Section~\ref{sec:audit} presents the \emph{statefulness audit}, \emph{longitudinal benchmark harness}, and counterfactual metrics. Section~\ref{sec:governance} examines failure modes, machine
unlearning, deployment considerations, and open research directions. Throughout this survey, we distinguish three forms of information access that are often conflated under the term ``memory.'' \textbf{\emph{Context}} denotes bounded, session-local working state, such as a Transformer's attention window or KV cache~\cite{vaswani2017attention,xiao2024efficient,kwon2023efficient}, that is created during inference and discarded afterward. \textbf{\emph{Retrieval}} denotes on-demand access to externally stored content through similarity search or graph traversal~\cite{lewis2020retrieval,karpukhin2020dense,edge2024local}. \textbf{\emph{Memory}}, in our framework, denotes persistent evidence-grounded document state: information that is retained, validated, consolidated, temporally versioned, and governed across sessions~\cite{packer2023memgpt,gutierrez2024hipporag,zhang2025survey}. Thus, context provides transient computation, retrieval provides access to external information, and memory provides persistent state. This distinction is essential for identifying the \emph{statelessness bottleneck}: a system may have long context or powerful retrieval without becoming stateful.

\section{Capability Evolution and the Access--Consolidation
Boundary}
\label{sec:evolution}

\begin{figure*}[!t]
\centering
\begin{tikzpicture}[font=\scriptsize,
  cell/.style={circle, fill=black!75, inner sep=1.5pt},
  target/.style={circle, fill=red!70!black, draw=black, line width=0.6pt, inner sep=2.2pt}]
\draw[->,thick] (0,0) -- (12.8,0) node[right,font=\small] {Document scope $\rightarrow$};
\draw[->,thick] (0,0) -- (0,8.4) node[above,font=\small] {State persistence $\rightarrow$};
\foreach \x/\l in {1.4/Region, 3.8/Page, 6.2/Multi-page, 8.6/Collection, 11.2/Stream}
  { \node[below] at (\x,-0.15) {\l}; \draw[gray!35] (\x,0) -- (\x,7.9); }
\foreach \y/\l in {0.8/None, 2.1/Session, 3.4/External store, 4.7/Episodic, 6.0/Semantic, 7.2/Governed}
  { \node[left,align=right] at (-0.15,\y) {\l}; \draw[gray!35] (0,\y) -- (12.4,\y); }
\fill[red!10] (0.3,5.45) rectangle (12.4,7.85);
\node[font=\scriptsize,red!55!black,anchor=north east,align=right] at (12.3,7.8)
  {empty at every scope:\\no consolidation, no governed forgetting};
\node[cell] at (3.8,0.8) {};
\node[above left,align=right] at (3.7,0.85) {DocVQA, ChartQA, FUNSD,\\ LayoutLM, Donut,\\ EAML, SelfDoc, UDoc,\\ VLCDoC, GlobalDoc};
\node[cell] at (6.2,0.8) {};
\node[above right] at (6.3,0.85) {MP-DocVQA, DUDE};
\node[cell] at (6.2,2.1) {};
\node[above right] at (6.3,2.15) {LongDocURL, M-LongDoc};
\node[cell] at (8.6,3.4) {};
\node[below right,align=left] at (8.7,3.25) {ColPali/ViDoRe, VisRAG,\\ GraphRAG, LightRAG};
\node[cell] at (11.2,4.7) {};
\node[below left,align=right] at (11.1,4.55) {MemGPT, LoCoMo,\\ SYNAPSE, GAAMA\\ (text-only)};
\node[target] at (11.2,6.6) {};
\node[left,align=right,font=\scriptsize\bfseries,red!55!black] at (10.85,6.6) {Target: $\mathcal{C}^{*}$};
\end{tikzpicture}
\caption{The capability map of Document AI under the placement
rubric of Section~\ref{sec:evolution}: horizontal axis = broadest
document scope; vertical axis = highest persistence regime. The
cross-modal representation lineage saturates at Page scope with no
persistence; retrieval systems reach External-store persistence at
Collection scope; text-only agent memory reaches Episodic
persistence over conversation streams, off the document-structure
axis. The Semantic and Governed band is empty at every scope; the
red target marks the conjunctive capability $\mathcal{C}^{*}$.
That emptiness is the \emph{statelessness bottleneck}.}
\label{fig:capmap}
\end{figure*}
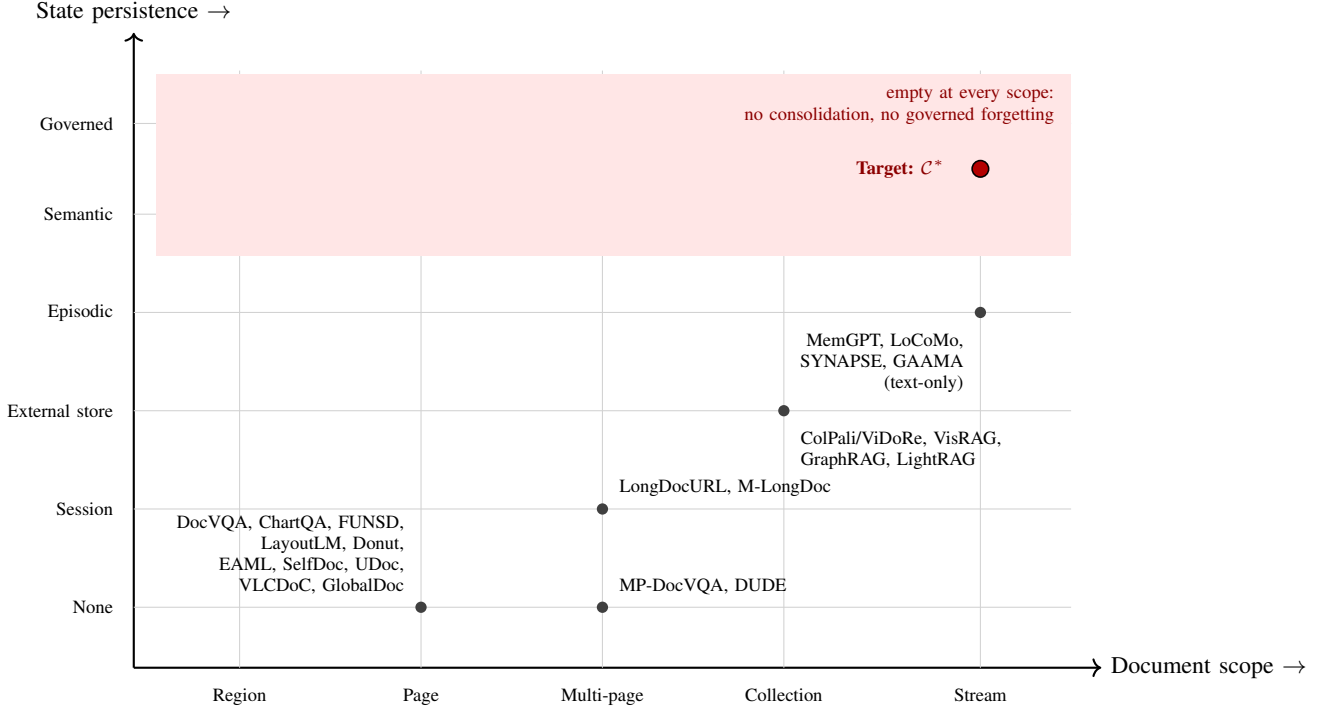

\paragraph{Placement rubric for Figure~\ref{fig:capmap}.}
Each system is placed at the coordinate $(x, y)$ where $x$ is the
broadest document scope it processes (Region, Page, Multi-page,
Collection, or Stream) and $y$ is the highest persistence regime
it achieves (None, Session, External store, Episodic, Semantic,
or Governed). Scope is determined by the input specification of
the system's primary evaluation benchmark. Persistence is
determined by whether the system's architecture retains state
across explicit session boundaries: ``None'' for stateless
functions; ``Session'' for long-context models that process
multiple pages within one inference call; ``External store'' for
retrieval-augmented systems with persistent indices; ``Episodic''
for agent-memory systems that retain cross-session experience
without document-native structure. No published system reaches
the Semantic or Governed bands under our operational definition.
Benchmarks are plotted by their evaluation protocol; architectures
by their design. Text-mediated agent abstractions (\eg, spreading
activation over concept graphs) remain at Episodic because they
discard document-native structure; the Semantic band requires
consolidation of visual and layout evidence into reusable schemas
with span-level provenance, which no cited system performs.

This section traces the capability evolution from single-page
perception to agentic reasoning and identifies the boundary where
accumulated capability stops and the \emph{statelessness bottleneck}
begins. Each lineage is examined for what it achieves within a single
episode and what it structurally cannot achieve across episodes.

\subsection{Perception, Recognition, and Layout Analysis}
\label{sec:perception}

Early document processing relied on hand-engineered pipelines of
binarization, segmentation, OCR, and template
matching~\cite{nagy2000twenty}; deep learning shifted the field
toward learned visual representations for page-level feature
extraction~\cite{harley2015evaluation} and layout
detection~\cite{binmakhashen2019document,clausner2013significance}.
Layout benchmarks subsequently scaled in size and diversity, from
PubLayNet~\cite{zhong2019publaynet} and
DocBank~\cite{li2020docbank} through
DocLayNet~\cite{pfitzmann2022doclaynet} and
D$^{4}$LA~\cite{da2023vision} to the multi-component
OmniDocBench~\cite{ouyang2025omnidocbench}. Text recognition
evolved from CTC-based sequence
models~\cite{graves2006connectionist} through
CRNN~\cite{shi2016end} to Transformer-based
OCR~\cite{li2023trocr}; handwritten text recognition followed a
parallel trajectory from recurrent
architectures~\cite{graves2008offline} to segmentation-free
document-level models~\cite{coquenet2023dan,mondal2024icdar};
and mathematical expression recognition introduced
two-dimensional structural
criteria~\cite{zanibbi2012recognition,mouchere2014icfhr,deng2017image}.
These perception and recognition capabilities provide inputs to
downstream document-understanding tasks. For this survey, the
important distinction is not that recognition or layout systems
are inherently stateless, but that their standard inference
formulation is typically episode-local. Unless an explicit
adaptation or memory mechanism is introduced, processing
additional pages does not by itself expose a persistent
operational state to subsequent episodes. Repeated inference and
stateful adaptation are therefore distinct capabilities, with the
latter forming the persistence dimension examined in the remainder
of this survey.

\subsection{Cross-Modal and Layout-Aware Representation}
\label{sec:representation}

Multimodal document representation evolved along complementary cross-modal and layout-aware lineages; the cross-modal lineage framed documents as retrieval and classification objects. EAML~\cite{bakkali2021eaml} introduced ensemble self-attention with mutual learning to jointly train image and text modalities; SelfDoc~\cite{li2021selfdoc} proposed self-supervised cross-modal representation learning; UDoc~\cite{Gu2022UnifiedPF} consumed multimodal embeddings jointly under losses explicitly aligning textual and visual modalities; Bi-VLDoc~\cite{luo2025bi} added bidirectional vision--language supervision; VLCDoC~\cite{bakkali2023vlcdoc} applied vision--language contrastive pretraining to cross-modal document classification; and GlobalDoc~\cite{bakkali2025globaldoc} culminated the line with unified vision--language objectives for real-world document image retrieval and classification. The thread is that semantic identity is a property of the \emph{joint} vision--language embedding, not of either modality alone; Doc2graph~\cite{gemelli2022doc2graph} extended the principle to task-agnostic document understanding.

The layout-aware lineage developed alongside it. LayoutLM~\cite{xu2020layoutlm} showed that injecting two-dimensional bounding-box coordinates into self-attention yields substantial gains on form understanding and information extraction — a design template refined for five years. LayoutLMv2 and LayoutLMv3~\cite{xu2021layoutlmv2,huang2022layoutlmv3} added visual patch features alongside text and spatial tokens; DiT~\cite{li2022dit} applied masked image modeling to document pages; DocFormer~\cite{appalaraju2021docformer} proposed a unified multi-modal architecture; StrucTexT, BROS, TILT, StructuralLM, and ERNIE-Layout~\cite{li2021structext,hong2022bros,powalski2021going,Li2021StructuralLM,peng2022ernie} explored alternative spatial encodings and pretraining objectives; and UDOP, DocLLM, and LayoutLLM~\cite{tang2023unifying,Wang2024DocLLM,Luo2024LayoutLLM} extended layout-aware modeling to generative LLMs. The common thread is that position is semantics. OCR-free architectures such as Donut~\cite{kim2022ocr} and Pix2Struct~\cite{lee2023pix2struct} eliminated the recognition stage, mapping raw page images to structured text through encoder-decoder transformers — bypassing token-level OCR error propagation and simplifying deployment, at the cost of character-level precision modular pipelines provide.

More importantly, the latent representations produced by the three lineages: cross-modal, layout-aware, and OCR-free, are typically transient. They exist during inference and, unless explicitly externalized, are discarded when processing ends. Standard systems do not retain a document encoding as an operational state for processing, nor do they use the encoding of one invoice to accelerate the parsing of another. Table~\ref{tab:reprlineage} makes this pattern explicit: across representation families from BERT to GlobalDoc to LayoutLLM, the native-state column is uniformly empty under our definition of persistent document state. The representation is computed and consumed within an episode; it does not, by itself, constitute persistent knowledge that can be validated, consolidated, versioned, or governed across sessions.

\begin{table*}[!t]
\centering
\caption{Representation lineage in Document AI. The final column
indicates whether the architecture natively maintains persistent
state across sessions. No representation family produces state that
survives the episode.}
\label{tab:reprlineage}
\small
\resizebox{\linewidth}{!}{%
\begin{tabular}{@{}lccccccc@{}}
\toprule
\textbf{Method} & \textbf{Year} & \textbf{Modality} &
\textbf{Pretraining objective} & \textbf{Layout} &
\textbf{Visual} & \textbf{OCR-free} & \textbf{Native state} \\
\midrule
BERT~\cite{devlin2019bert} & 2019 & Text & MLM & -- & -- & -- & -- \\
LayoutLMv1--v3~\cite{xu2020layoutlm,huang2022layoutlmv3} & 2020--22 & Text+Layout+Vision & MLM/MIM & \checkmark & \checkmark & -- & -- \\
EAML~\cite{bakkali2021eaml} & 2021 & Image+Text & Mutual learning & -- & \checkmark & -- & -- \\
SelfDoc~\cite{li2021selfdoc} & 2021 & Cross-modal & Self-supervised & \checkmark & \checkmark & -- & -- \\
UDoc~\cite{Gu2022UnifiedPF} & 2021 & Cross-modal & Unified alignment & \checkmark & \checkmark & -- & -- \\
DiT~\cite{li2022dit} & 2022 & Vision & MIM & Implicit & \checkmark & \checkmark & -- \\
Bi-VLDoc~\cite{luo2025bi} & 2022 & Cross-modal & Bidirectional VL & \checkmark & \checkmark & -- & -- \\
Donut~\cite{kim2022ocr} & 2022 & Vision$\to$Text & Seq2seq & Implicit & \checkmark & \checkmark & -- \\
VLCDoC~\cite{bakkali2023vlcdoc} & 2023 & Cross-modal & VL contrastive & -- & \checkmark & -- & -- \\
Pix2Struct~\cite{lee2023pix2struct} & 2023 & Vision$\to$Text & Seq2seq & Implicit & \checkmark & \checkmark & -- \\
UDOP~\cite{tang2023unifying} & 2023 & Vision+Text+Layout & Generative & \checkmark & \checkmark & Partial & -- \\
DocLLM / LayoutLLM~\cite{Wang2024DocLLM,Luo2024LayoutLLM} & 2024 & Vision+Text+Layout & Instruction tuning & \checkmark & \checkmark & Partial & -- \\
GlobalDoc~\cite{bakkali2025globaldoc} & 2025 & Cross-modal & Unified VL & \checkmark & \checkmark & -- & -- \\
\bottomrule
\end{tabular}
}%
\end{table*}

\subsection{Downstream Tasks and the Single-Shot Convergence}
\label{sec:tasks}

The tasks posed to these representations grew more demanding; their genealogy reveals a consistent structural property. Task scope expanded from single-page visual QA and information extraction~\cite{kembhavi2017you,singh2019towards,mishra2019ocr,biten2019scene,mathew2021docvqa,Mathew2022InfographicVQA,jaume2019funsd,borchmann2021due,stanislawek2021kleister,Tanaka2021VisualMRC,tito2021document}, to table and chart reasoning with arithmetic programs~\cite{Herzig2020TAPAS,Yin2020TaBERT,smock2022pubtables,zhu2021tat,Chen2021FinQA,Zhao2022MultiHiertt,Methani2020PlotQA,Masry2022ChartQA,Masry2025ChartQAPro,Chen2026ChartR,kahou2017figureqa,Kafle2018DVQA}, to scientific and multipage evidence aggregation~\cite{dasigi2021dataset,pramanick2024spiqa,Sundar2024cPAPERS,Li2024M3SciQA,tito2023hierarchical,van2023document,Blau2024GRAM,Cho2025M3DocVQA}, to long-context and corpus-level synthesis~\cite{Deng2025LongDocURL,chia2025m,wang2024needle,caciularu2021cdlm,Suri2025VisDoM,nourbakhsh2025coming}. At each step the unit of challenge grew — from reading a region, to understanding a page, to navigating a document, to acquiring evidence from a collection — yet the field's assumption remained unchanged: the document is a fixed input, the task is a single function, and the function is evaluated once. No benchmark in this genealogy asks whether the system is different after processing the document than it was before. The single-shot formulation is not an accident of any benchmark; it is the shared grammar of the task family.

\subsection{Long-Context Scaling and Efficient Architectures}
\label{sec:longcontext}

As document length grew, the field's engineering response was to extend the context window. Sparse and factorized attention mechanisms~\cite{beltagy2020longformer,zaheer2020big,child2019generating,Kitaev2020ReformerTE,dai2019transformer} reduced the quadratic cost of self-attention; position-encoding innovations~\cite{press2021train,su2024roformer,chen2023extending,peng2024yarn,ding2024longrope} extended extrapolation range; and FlashAttention~\cite{dao2022flashattention} and state-space models~\cite{gu2023mamba} reduced computational cost. Prompt compression methods~\cite{jiang2023llmlingua,chevalier2023adapting,mu2023learning} reduce the token count fed to the model while attempting to preserve task performance. These methods address the cost of long context but not its persistence.
Empirical evaluations consistently identify structural limits: models exhibit reduced accuracy when relevant evidence is placed in the middle of long sequences~\cite{liu2024lost}, and performance decreases as the proportion of irrelevant distractors increases~\cite{shi2023large}. Extending the context window enlarges working memory without creating persistent state; a million-token context functions as a larger scratchpad that clears when inference terminates.

\subsection{Retrieval, Multimodal RAG, and State Integration}
\label{sec:retrieval}

The retrieval revolution addressed the collection-level challenge without disturbing the single-shot assumption. RAG, introduced by Lewis \etal~\cite{lewis2020retrieval}, decomposed knowledge-intensive generation into retrieval and generation steps, conditioning generation on external evidence without parameter changes. REALM~\cite{guu2020retrieval} jointly pretrained retrieval and language modeling; RETRO~\cite{borgeaud2022improving} scaled retrieval-augmented pretraining to trillions of tokens; dense dual-encoder~\cite{karpukhin2020dense} and late-interaction~\cite{khattab2020colbert} retrieval improved passage quality; Atlas~\cite{izacard2023atlas} demonstrated few-shot learning with small language models; comprehensive surveys cover RAG architectures, evaluation, and failure modes~\cite{procko2024graph,gao2026scaling}.

Document AI adapted these mechanisms to visually rich inputs. ColPali~\cite{Faysse2025ColPali} retrieves page images through multi-vector late-interaction representations, preserving layout that text-only chunking discards; VisRAG~\cite{Yu2025VisRAG} integrates visual document representations into the retrieval--generation loop; VDocRAG~\cite{Tanaka2025VDocRAG} extends visual retrieval to open-domain document QA; multimodal document retrieval~\cite{dong2025mmdocir} and multimodal RAG~\cite{dong2026benchmarking} benchmarks evaluate page and layout granularities; layout-aware dynamic retrieval~\cite{sourati2026lad} constructs symbolic document graphs during ingestion and with inference-time navigation; graph-based retrieval~\cite{edge2024local,guo2025lightrag} constructs corpus-level abstractions from ingested documents; knowledge-graph prompting~\cite{wang2024knowledge} uses structured graph evidence for multi-document QA; retrieval evaluation frameworks~\cite{chen2024benchmarking,ru2024ragchecker,es2024ragas} diagnose retrieval and generation failures separately.

These systems advance evidence access; this survey's distinction is between access and integration. A ColPali~\cite{Faysse2025ColPali} store that has served ten thousand queries has grown but has not consolidated: the vectors are read, not learned from. GraphRAG~\cite{edge2024local} and LightRAG~\cite{guo2025lightrag} reorganize during ingestion, building community summaries and dual-level indices; yet ingestion-driven reorganization is not experience-conditioned consolidation: these indices are reshaped by what is added, not by how their content was reasoned about across queries, and carry no temporal validity, no provenance through abstraction, and no governed deletion. Retrieval solves the problem of finding the right passage; consolidation solves the problem of being changed by having found it.

\begin{table*}[t]
\centering
\caption{Adjacent persistent-state paradigms against the conjunctive
requirement $\mathcal{C}^{*}$. No existing paradigm satisfies all
six properties jointly.}
\label{tab:adjacent}
\small
\begin{tabular}{@{}lcccccc@{}}
\toprule
\textbf{Paradigm} & \textbf{Multimodal} & \textbf{Doc structure} & \textbf{Persistent} & \textbf{Temporal} & \textbf{Provenance} & \textbf{Governance} \\
\midrule
Continual learning & -- & -- & $\checkmark$ & -- & -- & -- \\
Agent memory & Partial & -- & $\checkmark$ & Partial & Partial & -- \\
Multimodal RAG & $\checkmark$ & $\checkmark$ & Ext.\ index & -- & Page & -- \\
Commercial LLM memory & Partial & -- & Context only & -- & -- & -- \\
\textbf{Target $\mathcal{C}^{*}$} & $\checkmark$ & $\checkmark$ & $\checkmark$ & $\checkmark$ & $\checkmark$ & $\checkmark$ \\
\bottomrule
\end{tabular}
\end{table*}
\subsection{Agents, Continual Learning, and Adjacent Persistent State}
\label{sec:agents}

The agentic paradigm extends retrieval by introducing within-episode
state. Systems that invoke tools, plan multi-step workflows, and
maintain scratchpads across reasoning
steps~\cite{yao2022react,schick2023toolformer} possess state that
static pipelines lack. Reflexion~\cite{shinn2023reflexion}
introduces verbal reinforcement learning for iterative refinement.
Voyager~\cite{wang2023voyager} demonstrates open-ended skill
accumulation. Adaptive retrieval agents such as
Self-RAG~\cite{asai2024self} and CRAG~\cite{yan2024corrective}
dynamically decide when and what to retrieve. SimpleDoc and
Doc-V*~\cite{Jain2025SimpleDoc,zheng2026doc} introduce
iterative page retrieval and coarse-to-fine visual reasoning for
multipage documents. These capabilities are valuable, yet agent
state remains strictly episodic: it persists within a task episode
and is discarded at episode termination. A scratchpad tracking
reasoning steps across five tool invocations differs fundamentally
in kind, not merely in scale, from state that tracks knowledge
across five hundred documents.

Adjacent fields have demonstrated that cross-session persistence is achievable. Differentiable memory~\cite{weston2014memory,sukhbaatar2015end,graves2016hybrid,wang2024memoryllm} proved that external state can integrate into neural computation, but not the document-specific requirements of multimodal evidence, spatial layout, temporal versioning, and provenance-linked consolidation. Continual learning~\cite{kirkpatrick2017overcoming,Rebuffi2017iCaRL,lopez2017gradient,parisi2019continual,delange2021continual,monaikul2021continual,sassioui2026amd} has studied sequential acquisition for over a decade but produces no inspectable or deletable document state. Persistent agent systems introduce cross-session state: Generative Agents store experiences, generate reflections, and retrieve them in later interactions~\cite{park2023generative}; MemGPT~\cite{packer2023memgpt} manages hierarchical memory inspired by operating-system virtual memory; MemoryBank~\cite{Zhong2024MemoryBank} introduces time-aware memory decay; HippoRAG~\cite{gutierrez2024hipporag} retrieves via hippocampal indexing theory with a knowledge-graph associative index; SYNAPSE~\cite{jiang2026synapse} couples episodic and semantic stores through spreading activation; GAAMA~\cite{paul2026gaama} adds a concept-mediated associative knowledge graph; Agent Workflow Memory~\cite{wang2024agent} induces reusable workflows from past experience; surveys map storage, reflection, and experience abstraction~\cite{zhang2025survey,luo2026storage}; Always-on agents~\cite{ding2026always} independently specify a
persistent-state record, a lifecycle, and an evaluation protocol
for long-running, text-mediated agent systems. Because that work
does not couple persistence to multimodal document evidence,
span-level provenance, bitemporal validity, or governed deletion,
it falls outside the document-native conjunctive requirement; and multimodal long-term memory agents~\cite{long2026seeing} integrate visual, auditory, and textual episodes. CoALA~\cite{sumers2024coala} provides a conceptual vocabulary for
agent memory (working, episodic, semantic, and procedural stores)
that aligns with our episodic--semantic split, but offers no formal
state tuple, no provenance or validity invariants, no deletion
guarantees, and no document grounding.
In the agent literature these mechanisms are termed ``memory''; under the introduction's taxonomy they implement cross-session persistence and associative retrieval, but not the document-native specification of Section~\ref{sec:formalization}: none couples persistence to multimodal document evidence, span-level provenance, bitemporal validity, and governed deletion. Temporal knowledge-graph memory for agents is developed in
Zep/Graphiti~\cite{rasmussen2025zep}; scalable
extract--consolidate--retrieve memory in
Mem0~\cite{chhikara2025mem0}; and multi-session longitudinal
memory evaluation in LongMemEval~\cite{wu2025longmemeval}.
None of these couples persistence to multimodal document
evidence, span-level provenance, bitemporal validity, or
governed deletion.

\subsection{Commercial Memory and State Operators}
\label{sec:boundary}

Commercial deployments introduced a related but distinct capability: major providers now offer ``memory'' features that persist facts and preferences across sessions — genuine advances in personal assistant state. We evaluate these systems by observable behavioral properties rather than speculating about proprietary backends: public documentation describes user-editable memory items and summaries surfaced into later interactions. Whatever the underlying implementation, these features share a common observable signature: they retain the fact but discard the evidence. They cannot trace a remembered liability cap back to the bounding box on page four of the contract from which it was extracted; nor distinguish a clause superseded by a later amendment from one that remains in force; nor execute a verifiable deletion when a data retention policy arrives. They are, in the vocabulary of this survey, \textit{persistent context}, not \textit{persistent memory}. The gap between personal assistant and institutional document memory is not engineering effort but architectural requirements not yet formulated.

Formulating those requirements begins with an operational separation the current literature conflates. \emph{Access} reads stored content without modifying state; it corresponds to retrieval via the associative index and leaves persistent state unchanged. \emph{Adaptation} changes model parameters through gradient updates on a document stream; it produces no inspectable or deletable state. \emph{Consolidation} transforms persistent state based on validated evidence, producing a state that alters future processing. These operators can be composed — a complete system will compose them — but must not be conflated: retrieving a stored passage and conditioning on it is access; fine-tuning weights on a stream is adaptation; extracting a reusable schema from ten thousand invoices, linking it to provenance, updating it when the vendor's template changes, and removing it when the relationship ends is consolidation. The first two operators are well studied; the third is the subject of this survey. Section~\ref{sec:formalization} formalizes these three operators as the components of the state-update function and the operator set that governs persistent state transitions.

The target capability is therefore not ``memory'' in the generic
sense. It is a specific intersection of requirements that no
existing system satisfies jointly. We define this intersection
as the conjunctive capability:
\begin{align}
\mathcal{C}^{*} ={}&
\text{multimodal evidence} \cap \text{document structure}
\cap \text{persistent} \notag\\[-1mm] \text{state} &\cap\; \text{temporal validity}
\cap \text{provenance} \cap \text{governance}.
\label{eq:conjunctive}
\end{align}

Continual learning satisfies persistence but not multimodal document structure or provenance. Agent memory satisfies persistence and partial provenance but not document-native layout, bitemporal versioning, or governed deletion. Multimodal
RAG satisfies document structure and evidence modality but not persistence, consolidation, or forgetting. Commercial LLM memory features satisfy persistence in the form of append-only text
summaries but lack document structure, temporal validity,
provenance, and governance. No existing paradigm satisfies all
six properties jointly; Table~\ref{tab:adjacent} makes this gap
explicit. The survey's hypothesis is not that each component is
missing from the literature. It is that the intersection lacks a
shared, document-native task formulation, representational
specification, and evaluation protocol.

\subsection{Positioning Against Existing Surveys}
\label{sec:positioning}

Several surveys cover portions of the landscape reviewed above,
and we defer to them for the ground they cover. Recognition and
layout analysis are surveyed
in~\cite{garrido2025handwritten,binmakhashen2019document}. Document VQA
and MLLM-based document understanding are surveyed
in~\cite{ding2026survey,wang2025document,subramani2020survey}.
Retrieval-augmented generation is surveyed
in~\cite{procko2024graph,gao2026scaling}. Agent memory is
surveyed in~\cite{zhang2025survey,luo2026storage}.
Continual learning is surveyed in~\cite{parisi2019continual}. We
cite these works as authority and do not re-review the material
they cover.

Our survey extracts from each lineage only the evidence that bears on persistence: whether the task formulation, the architecture, or the benchmark admits state that survives an episode. These surveys treat their respective strands in isolation. This survey's contribution is the synthesis that connects them: it treats persistence as a first-class axis across the full stack, audits benchmarks for statefulness, and formalizes the state tuple, lifecycle operations, and operator
set. To our knowledge, no existing Document AI survey performs this
unification: treating persistence as a first-class axis across the
full stack, formalizing the state tuple, lifecycle operations, and
operator set, and auditing ten representative benchmarks for
statefulness. Section~\ref{sec:formalization} provides the representational specification. Section~\ref{sec:audit} provides the evaluation protocol. Together, they define what it would mean for a
Document AI system to remember, and what it would take to prove
that it does.

\section{Formalization of Persistent Evidence-Grounded State}
\label{sec:formalization}

The intersection identified in the previous section requires a
precise definition of what persistent state means in the document
setting. ``Memory'' is an overloaded term in the machine-learning
literature, referring to mechanisms as different as recurrent and
differentiable hidden
states~\cite{graves2016hybrid,weston2014memory,sukhbaatar2015end},
transient key-value attention
caches~\cite{xiao2024efficient,kwon2023efficient}, episodic buffers
and agent scratchpads~\cite{park2023generative,packer2023memgpt},
knowledge integrated into parametric
weights~\cite{kirkpatrick2017overcoming,wang2024memoryllm}, and
non-parametric external
datastores~\cite{lewis2020retrieval,karpukhin2020dense,gutierrez2024hipporag}.
These mechanisms differ in persistence, inspectability, and
governability. We therefore use \emph{persistent evidence-grounded
document state} to denote a state with explicit content, relations,
temporal validity, provenance, and authorization, and we define the
operations and invariants required to maintain it.

\subsection{Document Stream and State Specification}
\label{sec:state_spec}

A document collection evolves as a stream. Documents arrive over
time, may revise or supersede one another, belong to document
families, and carry metadata that determines their authority and
validity. We represent the collection as an ordered sequence
$\mathcal{S} = (D_1, D_2, \dots, D_T)$ over discrete sessions
$t \in \{1, \dots, T\}$. Each document artifact $D_t$ is
represented by the six-component tuple:
\begin{equation}
D_t = (V_t,\; T_t,\; B_t,\; G_t,\; H_t,\; M_t)
\label{eq:doctuple}
\end{equation}
where $V_t$ denotes visual pixel tensors, $T_t$ textual tokens,
$B_t = \{b_i\}_{i=1}^{N_t}$ denotes the set of spatial bounding
boxes with each $b_i=(x_{1,i},y_{1,i},x_{2,i},y_{2,i})\in\mathbb{R}^4$,
$G_t$ structural relation graphs (e.g., table rows and columns,
figure--caption links, and form--field pairings), $H_t$ hierarchical
parse trees (e.g., sections, headings, and nesting), and $M_t$
metadata including document identity, version, timestamp, and
authorization policy. This representation is intentionally richer
than a pixel tensor alone: the meaning of a table depends on its
row--column structure, while the meaning of an amendment depends on
the version it supersedes.

The state maintained over this stream preserves the same
distinctions. We define the \emph{persistent evidence-grounded
document state} at session $t$ as:
\begin{equation}
S_t = (E_t,\; K_t,\; \Gamma_t,\; \mathcal{V}_t,\; P_t,\; A_t)
\label{eq:statetuple}
\end{equation}
where $E_t$ is the episodic evidence buffer containing high-fidelity,
source-tagged observations from processed documents; $K_t$ is the
consolidated semantic store containing schemas, templates, and
generalized abstractions extracted from repeated episodes; $\Gamma_t$
is the associative relation graph connecting episodic and semantic
content through structural, temporal, and analogical edges;
$\mathcal{V}_t$ is the temporal validity index, which records for
each stored claim the interval $[t_{\text{start}},\, t_{\text{end}}]$
during which it is authoritative; $P_t$ is the provenance ledger,
which links every stored conclusion through its consolidation history
to the source-document bounding boxes from which it was derived~\cite{buneman2001why,cheney2009provenance,moreau2013prov}; and
$A_t$ is the authorization and policy state governing who may read,
write, update, or delete each component. The graph $\Gamma_t$ and
validity index $\mathcal{V}_t$ are state-side structures, distinct
from the document-level graph $G_t$ and visual tensors $V_t$ in
Equation~\eqref{eq:doctuple}.

Equation~\eqref{eq:statetuple} is a specification rather than a
commitment to a particular implementation. The six components may be
realized by a knowledge graph, a hybrid database, a neural memory
module, or a combination of substrates. The defining requirement is
that the representation preserve the six roles and that the state
transitions below act on them jointly.

\begin{algorithm}[t]
\caption{Per-Session Persistent State Transition}
\label{alg:lifecycle}
\begin{algorithmic}[1]
\Require Prior state $S_{t-1}$ (with $S_0 = \emptyset$ for the
first session); document $D_t$; query $q_t$; optional deletion
request $X_{\mathrm{del}}$
\Ensure Updated state $S_t$; auditable answer $y_t$
\State \textbf{(Acquire)} $o_t \leftarrow \textsc{Observe}(D_t)$
    \Comment{extract $(V_t,T_t,B_t,G_t,H_t,M_t)$}
\State $c_t \leftarrow \textsc{Corroborate}(o_t,\,P_{t-1})$
    \Comment{prior supporting evidence}
\State \textbf{(Validate)} $e_t \leftarrow
    \textsc{Validate}(D_t,o_t,c_t,U_t)$
\If{$e_t \neq \emptyset$} \Comment{pre-write gate passed}
    \State \textbf{(Write)} $E_t \leftarrow E_{t-1}\cup\{(e_t,\tau_t,
        \mathrm{prov}(e_t))\}$
    \State \textbf{(Consolidate)} $(K_t,\Gamma_t)\leftarrow
        \textsc{Consolidate}(E_t,K_{t-1},\Gamma_{t-1})$
        \Comment{$\max_\alpha[\mathcal{J}(\alpha)-\lambda\mathcal{R}(\alpha)]$}
    \State \textbf{(Reconcile)} $(K_t,\mathcal{V}_t)\leftarrow
        \textsc{Reconcile}(e_t,K_t,\mathcal{V}_{t-1})$
    \State $P_t \leftarrow P_{t-1}\cup\mathrm{prov}(e_t)$
        \Comment{provenance invariant}
\Else
    \State $S_t \leftarrow S_{t-1}$ \Comment{state unchanged}
\EndIf
\State $A_t \leftarrow A_{t-1}$
    \Comment{authorization inherited; changed only by policy events}
\If{deletion request $X_{\mathrm{del}}$ received}
    \State \textbf{(Forget)} $S_t \leftarrow
        \mathcal{U}_{\mathrm{unl}}(S_t,\, X_{\mathrm{del}})$
        \Comment{event-driven; Algorithm~\ref{alg:forget}}
\EndIf
\State $y_t \leftarrow \textsc{Answer}(q_t,D_t,S_t)$
    \Comment{conditioned on persistent state}
\State \Return $S_t,\,y_t$
\end{algorithmic}
\end{algorithm}
\subsection{State Transitions and Lifecycle Operations}
\label{sec:transitions}

State evolves between sessions, and the transition determines what enters persistent state. A naive formulation writes the generated output directly into state: $S_t = g(S_{t-1}, D_t, \hat{y}_t)$, where $g$ is an arbitrary state-update function. This is unsafe because $\hat{y}_t$ may be incorrect: a hallucinated field value, a misread table cell, or an unsupported inference can become a durable state error contaminating later sessions — structural and self-reinforcing, beyond a single response. We require a pre-write validation gate, yielding the three-stage transition:

\begin{align}
o_t &= \operatorname{Observe}(D_t)
 \label{eq:observe}\\
c_t &= \operatorname{Corroborate}(o_t,\, P_{t-1})
\label{eq:corroborate}\\
e_t &= \operatorname{Validate}(D_t,\, o_t,\, c_t,\, U_t)
\label{eq:validate}\\
S_t &= \Phi(S_{t-1},\, e_t,\, \Omega_t)
\label{eq:update}
\end{align}
where $o_t$ is the document-complete multimodal observation extracted from $D_t$ — text, layout coordinates, table structures, and visual elements; it is document-conditioned, not query-conditioned, so state must capture full content for queries not yet posed at session $t$. Query $q_t$ conditions the \emph{answer} in the answer step of Algorithm~\ref{alg:lifecycle}, not the \emph{observation}; $e_t$ is validated evidence admitted by the gate after corroboration against the prior provenance ledger $P_{t-1}$ and uncertainty $U_t$; and $\Phi$ applies the operator set $\Omega_t = \{\mathcal{S}_{\mathrm{sup}},\, \mathcal{C}_{\mathrm{sch}}\} \cup \begin{cases} \{\mathcal{U}_{\mathrm{unl}}\} & \text{if deletion request } X_{\mathrm{del}} \text{ received}\\ \emptyset & \text{otherwise} \end{cases}$ to the prior state and admitted evidence. The corroboration context $c_t$ is retrieved from the prior
provenance ledger $P_{t-1}$, providing supporting or
contradicting evidence for $o_t$. When validation fails,
$e_t = \emptyset$ and the state is left unchanged. The epistemic confidence $U_t \in [0,1]$ quantifies the
model's certainty in the observation $o_t$, defined as the
calibrated complement of predictive entropy over extracted
fields: $U_t = 1 - H(o_t \mid D_t) / H_{\max}$, where
$H_{\max}$ is the maximum entropy over the extraction
vocabulary ($U_t = 0$: maximally uncertain; $U_t = 1$:
certain). Observations with $U_t$ below a threshold $\tau_U$ are rejected at the validation gate. The three operators defined below determine how evidence is consolidated, conflicting claims reconciled, and information removed; validation is a state-integrity condition, not an architectural refinement.

\begin{definition}[Transformation Operators $\Omega_t$]
\label{def:operators}
Each operator in $\Omega_t$ is a partial function on the state in
Equation~\eqref{eq:statetuple}, activated by the admitted evidence
$e_t$:
\begin{itemize}
  \item \textbf{Schema consolidation}
  $\mathcal{C}_{\mathrm{sch}}: (E_t, K_{t-1}, \Gamma_{t-1})
\rightarrow (K_t, \Gamma_t)$ abstracts   episodic traces into a consolidated schema by solving the utility--compression objective (Equation~\eqref{eq:consolidate}, defined below), while preserving recursive provenance links in $P_t$.
  \item \textbf{Temporal supersession}
  $\mathcal{S}_{\mathrm{sup}}: (K_t, \mathcal{V}_{t-1}, e_t)
  \rightarrow (K_t, \mathcal{V}_t)$ resolves conflicts between $e_t$
  and the post-consolidation knowledge store $K_t$ into one of the
  five reconciliation outcomes.
  Under \emph{supersede}, the superseded claim's validity interval
  is closed at the effective time $t_c$ in $\mathcal{V}_t$, and
  $e_t$ becomes authoritative thereafter.
  \item \textbf{Deterministic unlearning}
  $\mathcal{U}_{\mathrm{unl}}: (S_t, X_{\mathrm{del}})
  \rightarrow S_t'$ removes the deletion target $X_{\mathrm{del}}$
  (a subset of state elements specified by a deletion request) across
  all affected components and repairs consolidated structures that
  depend on it, producing a repaired state $S_t'$ subject to the
  Forgetting Fidelity and Coverage Retention guarantees formalized
  in Section~\ref{sec:guarantees}.
\end{itemize}
\end{definition}

\begin{figure*}[!t]
\centering
\begin{tikzpicture}[font=\small,
  op/.style={rectangle, draw, rounded corners=2pt, minimum width=1.8cm, minimum height=0.7cm, align=center, fill=gray!12},
  st/.style={rectangle, draw, rounded corners=2pt, minimum width=2.1cm, minimum height=0.85cm, align=center, fill=orange!14},
  arr/.style={-{Stealth[length=2.2mm]}, thick},
  parr/.style={-{Stealth[length=2.2mm]}, thick, dashed, teal}]
\node[op] (acq) at (0,0) {Acquire};
\node[op] (val) at (2.3,0) {Validate};
\node[op] (wri) at (4.6,0) {Write};
\node[op] (con) at (6.9,0) {Consolidate};
\node[op] (upd) at (9.2,0) {Update};
\node[op] (for) at (11.5,0) {Forget};
\draw[arr](acq)--(val);
\draw[arr](val)--(wri);
\draw[arr](wri)--(con);
\draw[arr](con)--(upd);
\draw[arr](upd)--(for);
\node[st] (epi) at (4.6,-1.7) {Episodic\\Store $E_t$};
\node[st] (sem) at (6.9,-1.7) {Semantic\\Store $K_t$};
\node[st] (idx) at (9.2,-1.7) {Associative\\Index $\Gamma_t$};
\node[st] (tmp) at (11.5,-1.7) {Temporal\\Layer $\mathcal{V}_t$};
\node[st] (auth) at (13.8,-1.7) {Authorization\\State $A_t$};
\draw[arr](wri)--(epi);
\draw[arr](con)--(sem);
\draw[arr](con)--(idx);
\draw[arr](upd)--(tmp);
\node[st, minimum width=15.8cm, fill=teal!10] (prov) at (6.95,-3.2)
  {Provenance Ledger $P_t$ --- invariant across all operations};
\draw[parr](epi)--(prov);
\draw[parr](sem)--(prov);
\draw[parr](idx)--(prov);
\draw[parr](tmp)--(prov);
\draw[parr](auth)--(prov);
\end{tikzpicture}
\caption{The persistent document memory architecture. Six lifecycle
operations act on the six components of $S_t$
(Equation~\eqref{eq:statetuple}). Provenance ($P_t$) is an invariant
across all operations. The authorization state $A_t$ governs access
control. Forget removes the target across all stores and repairs
derived abstractions.}
\label{fig:architecture}
\end{figure*}
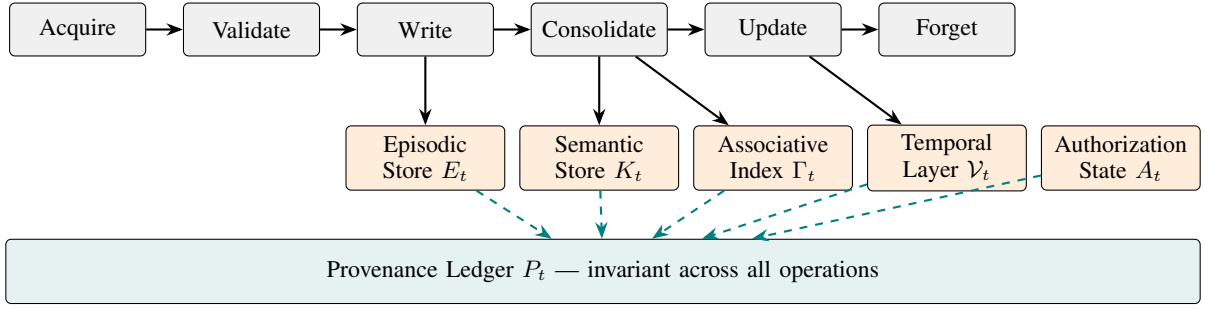

The operators transform state; the lifecycle governs their exercise in six operations: \emph{acquire}, \emph{validate}, \emph{write}, \emph{consolidate}, \emph{update and reconcile}, and \emph{forget}. \emph{Acquire} extracts the multimodal observation $o_t$ (text, layout, tables, visual elements). \emph{Validate} applies the gate in Equation~\eqref{eq:validate} before any persistent write. \emph{Write} appends admitted evidence $e_t$ to $E_t$, tagged with timestamp $\tau_t$ and provenance $P_t$. \emph{Consolidate} transforms episodic traces into generalized
schemas $K_t$ and relational structures $\Gamma_t$ according to the
utility--compression objective
\begin{equation}
\max_{\alpha}\;\bigl[\,\mathcal{J}(\alpha) - \lambda\,\mathcal{R}(\alpha)\,\bigr]
\label{eq:consolidate}
\end{equation}
where $\alpha$ ranges over candidate abstractions,
$\mathcal{J}(\alpha)$ is expected future-task utility,
$\mathcal{R}(\alpha)$ is storage or retrieval cost, and
$\lambda \ge 0$ controls the trade-off. The optimum need not be the smallest representation: episodic details remain valuable for tail queries when a semantic summary suffices for common ones. \emph{Update and reconcile} applies temporal supersession and one of five outcomes — \emph{support}, \emph{supersede}, \emph{scope-coexist}, \emph{retain-conflict}, \emph{escalate-for-review} — not binary overwriting: an amendment supersedes an earlier clause only after its effective date; two scientific papers may report different results for different populations; a preliminary financial statement may coexist with a later audited version because the claims carry different authority. \emph{Forget} executes verifiable removal of specified information across episodic, semantic, and index representations, including repair of abstractions that depended on it. Provenance is an invariant rather than a seventh operation: every operation preserves the chain from a stored conclusion to its source bounding boxes, through consolidation, supersession, and repair.
Algorithm~\ref{alg:lifecycle} collects these stages into a single
per-session procedure, making explicit how the validation gate, the
operator set, and the provenance invariant interact during each
state transition.

\subsection{Architecture and Degenerate Regimes}
\label{sec:architecture}
The corresponding architecture realizes the six components of
Equation~\eqref{eq:statetuple} according to complementary design
principles. Figure~\ref{fig:architecture} depicts this architecture:
six lifecycle operations act on the six state components, with
provenance maintained as an invariant across all operations.
Following complementary learning
systems~\cite{mcclelland1995there}, the episodic store ($E_t$) is
optimized for fast, high-fidelity writes, while the semantic store
($K_t$) is optimized for slow, compressed consolidation. The
associative index ($\Gamma_t$) supports retrieval by structural
analogy, temporal contiguity, and schema-driven expectation rather
than embedding similarity alone~\cite{gutierrez2024hipporag}. The
temporal layer ($\mathcal{V}_t$) distinguishes publication, event,
effective, and ingestion time~\cite{snodgrass1985taxonomy,snodgrass2012tsql2}. The provenance ledger ($P_t$) is
maintained as an immutable record, and the authorization state
($A_t$) enforces tenant and policy isolation.

These components are complementary rather than interchangeable: the episodic store preserves fidelity without generalization; the semantic store generalizes while abstracting source detail; the associative index improves access without governance; the temporal layer adds validity and ordering without encoding evidence; the provenance ledger adds traceability without retrieval utility; the authorization state adds governance without evidential content. The specification derives value from their composition.

Current Document AI systems are degenerate regimes retaining only subsets of the specification. The \emph{zero-state regime} sets $S_t = \emptyset$, omitting persistent state entirely: single-page benchmarks, long-context models, and standard Document QA reduce $P(Y \mid D_t, S_{t-1})$, where $Y$ denotes the task output, to $P(Y \mid D_t)$ because $S_{t-1}$ is absent. The \emph{monotonic append-only regime} sets $S_t = S_{t-1} \cup \operatorname{Embed}(D_t)$, appending embeddings to a vector index without consolidation, conflict resolution, or forgetting, as in standard RAG and visual-retrieval pipelines. Because $\Omega_t = \emptyset$ in Equation~\eqref{eq:update}, the index can grow without bound, retrieval interference may increase with corpus size, and no mechanism exists to revise or remove stored content. The \emph{unstructured summarization regime} sets $S_t = \operatorname{Decay}(S_{t-1}) \cup \operatorname{Summarize}(D_t)$, retaining compressed textual summaries while discarding spatial layout $B_t$, visual content $V_t$, and bounding-box provenance $P_t$; text-oriented agent-memory systems commonly approximate this regime. The full specification instead requires all six state components, the six lifecycle operations, and the provenance invariant — a conjunction no published system we identified satisfies under our operational definition.

\subsection{Formal Guarantees: Memory Harm and Governed Forgetting}
\label{sec:guarantees}

The pre-write validation gate in Equation~\eqref{eq:validate}
constitutes a formal integrity requirement rather than only an
architectural principle. The following definitions and bounds
quantify the risk of violating it.

\begin{definition}[Memory Harm]
\label{def:mh}
Let the task quality be the expected log-likelihood of the
ground-truth output $Y^*$:
\[
Q(D \mid S) \triangleq \mathbb{E}_{Y^*}\!\bigl[
\log P(Y^* \mid D, S)\bigr].
\]
Let $S_k^{\text{clean}}$ denote a contamination-free baseline
state. The Memory Harm at exposure $k$ is
\begin{equation}
\label{eq:mh}
MH(k) \triangleq Q(D_{k+1} \mid S_k^{\text{clean}})
- Q(D_{k+1} \mid S_k),
\end{equation}
where $MH(k) > 0$ indicates state-induced degradation.
\end{definition}

\begin{definition}[Binary reliance model]
\label{def:reliance}
The system is \emph{reliant} on a contaminated state item when
the item is retrieved and its content is used as if correct.
Given reliance, the task's expected log-likelihood decreases by
$D_{\mathrm{KL}}(P_{\mathrm{true}}\,\|\,P_{\mathrm{false}})$,
where $P_{\mathrm{true}}$ and $P_{\mathrm{false}}$ are the
model's output distributions conditioned on correct and corrupted
state, respectively. The weight $w_m \in [0,1]$ is the
probability that a contaminated state triggers reliance at the
evaluation point, assumed homogeneous across contamination
events for the bound below.
\end{definition}

\begin{theorem}[Contamination Accumulation and Expected Memory Harm]
\label{thm:mh}
Let $p_e \in (0,1)$ be the probability that an incoming document
observation contains an error, and let $p_c \in (0,1)$ be the
validation gate's conditional false-accept rate,
$p_c = P(\mathrm{consolidate} \mid \mathrm{erroneous})$. Assume
ingestion steps are independent, i.e., the per-step events
$\{\mathrm{error}_i \wedge \mathrm{consolidate}_i\}$ are independent
across $i \in \{1,\dots,k\}$. Then the probability that at least one
erroneous observation is consolidated within $k$ steps is
\begin{equation}
P_{\mathrm{cont}}(k) = 1 - (1 - p_e\, p_c)^k,
\label{eq:pcont}
\end{equation}
and, under the binary reliance model of
Definition~\ref{def:reliance}, the expected Memory Harm (\ie expectation over the contamination event,
with $MH(k)$ as defined in Definition~\ref{def:mh}) is
\begin{equation}
\mathbb{E}[MH(k)] = w_m \cdot P_{\mathrm{cont}}(k) \cdot
D_{\mathrm{KL}}\!\bigl(P_{\mathrm{true}} \,\|\, P_{\mathrm{false}}\bigr).
\label{eq:mhbound}
\end{equation}
\end{theorem}

\begin{remark}[Gate discriminativity]
\label{rem:gate}
No independence between error status and gate decision within a
step is assumed: by the definition of conditional probability,
$P(\mathrm{error} \wedge \mathrm{consolidate}) = p_e\, p_c$ holds
for any gate, discriminative or not. A useful validation gate is
precisely one whose false-accept rate $p_c$ is far below its
true-accept rate $P(\mathrm{consolidate} \mid \mathrm{correct})$;
the theorem quantifies the residual risk carried by the false
accepts that remain.
\end{remark}

\begin{remark}[Independence vs.\ family dependence]
The independence assumption in Theorem~\ref{thm:mh} concerns
ingestion steps (the per-step error-and-accept events); within a
step, no independence between error status and gate decision is
assumed (Remark~\ref{rem:gate}). It is distinct from the
family-level dependence of task outcomes discussed in
Section~\ref{sec:statistical}, which motivates the mixed-effects
estimator for evaluation. If contamination events are positively
correlated within a document family, the true contamination
probability at fixed $k$ is lower than Equation~\eqref{eq:pcont},
so Equation~\eqref{eq:mhbound} overstates expected harm; the
independent-step model therefore yields a conservative risk
statement for correlated failure modes and an exact one for
idiosyncratic errors.
\end{remark}

\begin{proof}
For step $i$, $P(\mathrm{error}_i \wedge \mathrm{consolidate}_i)
= P(\mathrm{error}_i)\, P(\mathrm{consolidate}_i \mid
\mathrm{error}_i) = p_e\, p_c$ by the definition of conditional
probability; no within-step independence is used. Independence
across steps then gives $P(\text{clean after } k) =
(1 - p_e\, p_c)^k$, yielding Equation~\eqref{eq:pcont}. Under the
binary reliance model the realized harm is $w_m\,
D_{\mathrm{KL}}(P_{\mathrm{true}} \| P_{\mathrm{false}})$ when the
state is contaminated and $0$ otherwise; taking the expectation
over the contamination event, which occurs with probability
$P_{\mathrm{cont}}(k)$, yields Equation~\eqref{eq:mhbound}.
\end{proof}

\begin{table*}[!t]
\centering
\caption{Candidate substrates for the six components of $S_t$.
Cells record whether the substrate natively provides each property,
under the coding rules stated in Section~\ref{sec:substrates}.
Columns: \emph{Persist} (survives sessions), \emph{Inspect}
(human-readable), \emph{Relate} (relational structure),
\emph{Generalize} (abstraction), \emph{Delete} (verifiable removal),
\emph{Ground} (bounding-box traceability). No single substrate
dominates; hybrid composition is required.}
\label{tab:substrates}
\small
\begin{tabular}{@{}lcccccc@{}}
\toprule
\textbf{Substrate} & Persist & Inspect & Relate & Generalize & Delete & Ground \\
\midrule
Vector store & $\checkmark$ & $\checkmark$ & Partial & Partial & $\checkmark$ & Partial \\
Knowledge graph & $\checkmark$ & $\checkmark$ & $\checkmark$ & Partial & $\checkmark$ & $\checkmark$ \\
Document database & $\checkmark$ & $\checkmark$ & $\checkmark$ & -- & $\checkmark$ & $\checkmark$ \\
Parametric update & $\checkmark$ & -- & Implicit & $\checkmark$ & -- & -- \\
Episodic memory & $\checkmark$ & $\checkmark$ & Partial & -- & $\checkmark$ & $\checkmark$ \\
Hybrid & $\checkmark$ & $\checkmark$ & $\checkmark$ & $\checkmark$ & $\checkmark$ & $\checkmark$ \\
\bottomrule
\end{tabular}
\end{table*}
\begin{corollary}[Validation Gate Controls Contamination Accumulation]
\label{cor:mh}
With $p_c$ the conditional false-accept rate of
Theorem~\ref{thm:mh} (inherited, not redefined), a perfect gate
($p_c = 0$) gives $P_{\mathrm{cont}}(k) = 0$ and hence
$\mathbb{E}[MH(k)] = 0$ for all $k$. For any fixed $p_c > 0$,
$P_{\mathrm{cont}}(k) \to 1$ as $k \to \infty$: the gate controls
the accumulation rate, not the asymptote. For a target
contamination level $\delta \in (0,1)$, the state satisfies
$P_{\mathrm{cont}}(k) \le \delta$ over the horizon
$k \le \ln(1-\delta)/\ln(1 - p_e\, p_c)$. This guarantee is
conditional on the supplied $p_c$: the corollary neither derives
nor certifies the false-accept rate of any concrete gate, which
must be measured empirically, e.g., by controlled corruption
injection as specified in Section~\ref{sec:contamination}.
\end{corollary}

Forgetting admits a parallel formal treatment. Let $X_{\mathrm{del}}$
denote a target evidence subset requested for erasure, and let
$Y_{\mathrm{ret}}$ denote un-targeted retained evidence. Recall that
task quality is the expected log-likelihood
$Q(\cdot) = \mathbb{E}[\log P(Y^* \mid \cdot)]$, which is
non-positive and higher-is-better; Coverage Retention therefore
requires that the \emph{drop} in retained quality be bounded,
$Q(Y_{\mathrm{ret}} \mid S_t') \ge Q(Y_{\mathrm{ret}} \mid S_t) -
\varepsilon$. Define deletion leakage as the mutual information between the
target and the post-deletion state,
$\operatorname{Leak}(X_{\mathrm{del}} \mid S_t')
\triangleq I(X_{\mathrm{del}};\, S_t')$, so that
$\operatorname{Leak} = 0$ if and only if $S_t'$ is independent
of $X_{\mathrm{del}}$. In practice $I(X_{\mathrm{del}};S_t')$
is estimated by a probe-based upper bound
(Algorithm~\ref{alg:forget}, Level~5): the advantage of the best
membership- or attribute-inference adversary over chance.
Forgetting Fidelity is therefore defined relative to the probe
class. In sign-safe form we define Forgetting Fidelity and
Coverage Retention as:
\begin{align}
FF &\triangleq 1 - \frac{\operatorname{Leak}(X_{\mathrm{del}} \mid S_t')}
{\operatorname{Leak}(X_{\mathrm{del}} \mid S_t)}
\label{eq:ff}\\
CR &\triangleq \exp\!\bigl( Q(Y_{\mathrm{ret}} \mid S_t') -
Q(Y_{\mathrm{ret}} \mid S_t) \bigr)
\label{eq:cr}
\end{align}
where $S_t'$ is the state after executing the unlearning operator
$\mathcal{U}_{\mathrm{unl}}$, and $\operatorname{Leak}(X_{\mathrm{del}}
\mid S_t)$ is the pre-deletion leakage. Conventions and ranges:
$FF \le 1$ always; $FF = 1$ indicates complete removal; $FF = 0$
indicates no reduction in leakage; $FF < 0$ indicates the repair
\emph{increased} leakage and is reported as such; when
$\operatorname{Leak}(X_{\mathrm{del}} \mid S_t) = 0$ the request is
vacuous and $FF$ is reported as $1$. For Coverage Retention,
$CR \in (0,\infty)$: $CR = 1$ indicates deletion left retained
knowledge untouched; $CR < 1$ indicates retention loss; $CR > 1$
indicates the repair improved retained quality, which is
permissible and indicates the deleted content was interfering with
retained reasoning. The guarantee $CR \ge 1 - \epsilon$ therefore
bounds the per-instance log-likelihood budget at
$-\ln(1-\epsilon)$ nats. In empirical reporting, $Q$ may be
instantiated by any higher-is-better quality measure
(log-likelihood, accuracy, or ANLS): both $EG$ (a difference) and
$CR$ (an exponentiated difference) are sign-safe under any such
instantiation, while Theorem~\ref{thm:mh} requires the
log-likelihood instantiation specifically.

Suppose document $D_j$ was consolidated into a schema $\sigma = \mathcal{C}_{\text{sch}}(E_T, K_{T-1}, \Gamma_{T-1})$ where $E_T$ includes episodic traces from $D_j$. Deleting the source file and its index pointer without repairing $\sigma$ leaves residual information: $\operatorname{Leak}(D_j \mid S_t^{\text{unrepaired}}) = I(D_j;\, \sigma) > 0$, because $D_j$ contributed to $\sigma$; an adversary probing $\sigma$ can partially reconstruct attributes of $D_j$. Perfect Forgetting Fidelity ($FF = 1$) requires schema repair when a consolidated abstraction $\sigma \in K_t$ satisfies $I(X_{\mathrm{del}};\, \sigma) > 0$. If the deleted evidence exists only in $E_t$ and has not been
consolidated, Levels~1--2 suffice. If $D_j$ has been consolidated
into $\sigma$, Level~3 schema repair is required: re-consolidating
$\sigma' = \mathcal{C}_{\text{sch}}(E_T \setminus E_{D_j},
K_{T-1}, \Gamma_{T-1})$ removes the direct contribution of
$E_{D_j}$, but does not by itself guarantee $I(D_j;\,\sigma')=0$,
because $K_{T-1}$, $\Gamma_{T-1}$, and $P_T$ may each retain
information derived from $D_j$ through earlier consolidation
steps. We therefore state the sufficient condition as a
conjecture.
\begin{conjecture}[Schema-repair sufficiency]
\label{conj:schema}
If (i)~$\sigma$ is a function of $E_T$ alone,
(ii)~no other state component ($K_{T-1}$, $\Gamma_{T-1}$,
$P_T$, $A_t$) encodes information derived from $D_j$, and
(iii)~re-consolidation is deterministic and stable under removal
of $E_{D_j}$, then $I(D_j;\,\sigma')=0$ and
$Q(Y_{\mathrm{ret}}\mid S_t') \ge Q(Y_{\mathrm{ret}}\mid S_t)
- \varepsilon$.
\end{conjecture}
Conditions~(i)--(iii) do not hold for a general consolidated
store; verifying them for a specific implementation is an open
problem. Deletion in a persistent document system is therefore a five-level hierarchy: physical object removal, representation invalidation, consolidated schema repair, parametric unlearning, and certified non-recoverability. One level does not imply the next: the first two are standard database operations; the third requires repairing information already absorbed into a consolidated representation.

\subsection{Substrate Composition}
\label{sec:substrates}

The formalization leaves one implementation question: what
substrates should instantiate $S_t$? No single substrate provides
all required properties. A vector store offers efficient semantic
retrieval with limited explicit temporal and relational semantics. A
knowledge graph provides relations and provenance but can be
expensive to construct and maintain from multimodal evidence. A
symbolic document store preserves source fidelity with limited
generalization. Parameter adaptation integrates knowledge compactly
into model behavior but complicates attribution and selective
deletion. A hybrid architecture follows naturally from the state
specification: because no single substrate satisfies all six
required properties, the specification deliberately leaves the
choice of substrate open. Any combination of storage engines,
graph structures, differentiable modules, or future
representations not yet invented can be evaluated against the
specification, provided the resulting system preserves the six
roles, supports the six lifecycle operations, and maintains the
provenance invariant. The benchmark harness in the next section
evaluates implementations against those requirements without
privileging a particular substrate.

\noindent\textbf{Coding rules for Table~\ref{tab:substrates}.}
Cells record whether the substrate \emph{natively} provides the property:
``$\checkmark$'' = native support; ``Partial'' = achievable with additional engineering; ``Implicit'' = emergent behavior only; ``--'' = not achievable by design. Each row reflects the substrate's published design specification.

\section{The Statefulness Audit and Longitudinal Benchmark}
\label{sec:audit}

\begin{table*}[!t]
\centering
\caption{Statefulness audit of representative Document AI
benchmarks. ``Partial'' indicates a
neighboring capability without the full persistent-state criterion. The LoCoMo/SYNAPSE row represents adjacent-field (text-only
conversational) evaluation. The final row is the specification
from Section~\ref{sec:formalization}, not an empirical result.}
\label{tab:audit}
\small
\begin{tabular}{@{}lcccccccc@{}}
\toprule
\textbf{Benchmark} & \textbf{Sequential} & \textbf{Cross-doc}
& \textbf{Family} & \textbf{Temporal} & \textbf{Provenance}
& \textbf{Forgetting} & \textbf{Session} & \textbf{Answer.} \\
\midrule
DocVQA~\cite{mathew2021docvqa}
  & -- & -- & -- & -- & Partial & -- & -- & -- \\
MP-DocVQA~\cite{tito2023hierarchical}
  & -- & -- & -- & -- & Page & -- & -- & -- \\
DUDE~\cite{van2023document}
  & -- & -- & -- & -- & Page & -- & -- & Yes \\
LongDocURL~\cite{Deng2025LongDocURL}
  & -- & -- & -- & -- & Box & -- & -- & -- \\
M-LongDoc~\cite{chia2025m}
  & -- & Partial & -- & -- & Partial & -- & -- & -- \\
ViDoRe~\cite{Faysse2025ColPali}
  & -- & -- & -- & -- & Page & -- & Ext.\ index & -- \\
ChartQA~\cite{Masry2022ChartQA}
  & -- & -- & -- & -- & Element & -- & -- & -- \\
FUNSD~\cite{jaume2019funsd}
  & -- & -- & -- & -- & Box & -- & -- & -- \\
PubTables-1M~\cite{smock2022pubtables}
  & -- & -- & -- & -- & Cell & -- & -- & -- \\
QASPER~\cite{dasigi2021dataset}
  & -- & -- & -- & -- & Section & -- & -- & Yes \\
\midrule
LoCoMo~\cite{maharana2024evaluating}
/ SYNAPSE~\cite{jiang2026synapse}
  & Yes & -- & Partial & Partial & Text-only & -- & Yes & Yes \\
\midrule
\textbf{Proposed harness}
  & \textbf{Yes} & \textbf{Yes} & \textbf{Yes} & \textbf{Yes}
  & \textbf{Span} & \textbf{Yes} & \textbf{Yes} & \textbf{Yes} \\
\bottomrule
\end{tabular}
\end{table*}

The formalization in the previous Section~\ref{sec:formalization} defines what persistent document state must contain and what its lifecycle operations must achieve, but these definitions require an evaluation protocol to assess whether they are satisfied in practice. Such an evaluation requires two complementary instruments: an audit of existing benchmarks that characterizes the current coverage and limitations of the field, and a new benchmark harness that makes the remaining requirements measurable. We describe both.

\subsection{The Statefulness Audit}
\label{sec:audit_existing}
Under this protocol, we audited the evaluation suites that
define the Document AI landscape, drawing on a search corpus
bounded between 2017 and 16 September 2026. Each benchmark
was coded against eight operational criteria: sequential
presentation, cross-document dependency, repeated-family
benefit, temporal revision, provenance persistence,
verifiable forgetting, session persistence, and answerability.
The audited suites span the major task families that the
community has built over this period, and each was
constructed to solve a specific problem that preceded it.

DocVQA~\cite{mathew2021docvqa} standardized single-page visual reading comprehension; MP-DocVQA~\cite{tito2023hierarchical} extended the spatial scope to multipage documents; DUDE~\cite{van2023document} introduced unanswerable questions forcing abstention; LongDocURL~\cite{Deng2025LongDocURL} pushed context to $50$--$150$ pages; M-LongDoc~\cite{chia2025m} exposed evidence acquisition as a bottleneck; ViDoRe~\cite{Faysse2025ColPali} shifted retrieval evaluation to visual pages; ChartQA~\cite{Masry2022ChartQA}, FUNSD~\cite{jaume2019funsd}, PubTables-1M~\cite{smock2022pubtables}, and QASPER~\cite{dasigi2021dataset} cover chart reasoning, key information extraction, table structure, and scientific QA. Each benchmark solved a real problem the field needed solved. None asks whether the system is different after processing the document than it was before: no cross-session dependency, no temporal revision, no consolidation reward, no verifiable forgetting.

\noindent\textbf{Coding rules for Table~\ref{tab:audit}.}
Each cell records the \emph{maximum capability exercised by the benchmark protocol}: ``Yes'' = protocol requires the capability with ground-truth labels; ``--'' = neither required nor measured; ``Partial'' = neighboring capability without the full persistent-state criterion. The \textbf{Provenance} column records annotation-grounding granularity within a single episode (Page, Box, Cell, Element, Section, Span, or Text-only); all audited benchmarks score ``--'' on provenance \emph{persistence} as defined in
Section~\ref{sec:formalization}. The \textbf{Session} column records explicit
session boundaries; ``Ext.\ index'' denotes retrieval over an external corpus
without session-level state persistence. Benchmarks absent from the table
were excluded as subsumed by a listed benchmark in the same task family
or lacking a standardized protocol.

The pattern across all audited suites is unambiguous. Every benchmark evaluates within-episode perception: visual reading, layout parsing, evidence retrieval, or long-context reasoning over a fixed input set. No benchmark in the audited corpus
presents documents sequentially across explicit session boundaries, requires a future answer to depend on evidence processed in a prior session, tests whether repeated exposure to a document family improves quality or reduces cost, evaluates temporal supersession or version-aware conflict resolution, or measures verifiable forgetting. The last row of Table~\ref{tab:audit} is not an empirical result. It is the specification from Section~\ref{sec:formalization} rendered as evaluation criteria. The gap between the last row and every other row is the \emph{statelessness bottleneck}, made visible as an absence in a table.

\begin{figure*}[!t]
\centering
\begin{tikzpicture}[font=\small,
  blk/.style={rectangle, draw, rounded corners=2pt, minimum width=2.6cm, minimum height=0.85cm, align=center, fill=gray!12},
  wrm/.style={rectangle, draw, rounded corners=2pt, minimum width=2.4cm, minimum height=0.8cm, align=center, fill=orange!16},
  cld/.style={rectangle, draw, rounded corners=2pt, minimum width=2.4cm, minimum height=0.8cm, align=center, fill=blue!10},
  res/.style={rectangle, draw, rounded corners=2pt, minimum width=1.9cm, minimum height=0.7cm, align=center, fill=green!10},
  arr/.style={-{Stealth[length=2.2mm]}, thick}]
\node[blk] (str) at (0,0) {Stream $D_1\dots D_k$\\(session-bounded)};
\node[wrm] (wrm) at (3.6,1.0) {Warm state $S_k$};
\node[cld] (cld) at (3.6,-1.0) {Cold baseline\\$S_0 = \emptyset$};
\node[blk] (tsk) at (7.0,0) {Future task\\$(D_{k+1}, q_{k+1})$};
\node[res] (qw) at (10.0,1.0) {$Q_{warm}$};
\node[res] (qc) at (10.0,-1.0) {$Q_{cold}$};
\node[res, fill=red!12] (dq) at (12.8,0) {$\Delta_Q$};
\draw[arr](str)--(wrm);
\draw[arr](wrm)--(tsk); \draw[arr](cld)--(tsk);
\draw[arr](tsk)--(qw); \draw[arr](tsk)--(qc);
\draw[arr](qw)--(dq); \draw[arr](qc)--(dq);
\end{tikzpicture}
\caption{The longitudinal benchmark harness. The warm path processes
$D_1\dots D_k$ and retains $S_k$; the cold baseline starts from
$S_0=\emptyset$. Both are evaluated on the identical future task
$(D_{k+1}, q_{k+1})$. The difference $\Delta_Q$ (Experience Gain)
isolates the contribution of persistent state. The remaining four
metrics are computed from related conditions of this harness.}
\label{fig:harness}
\end{figure*}
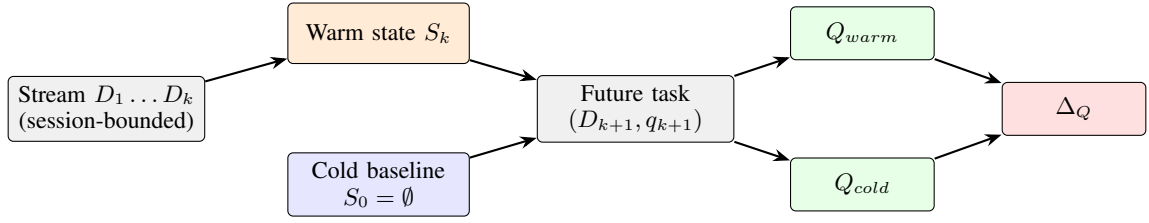

\subsection{The Longitudinal Benchmark Harness}
\label{sec:harness}
The absence is not accidental. Because benchmarks do not reward persistence, researchers do not build it; because researchers do not build it, benchmarks do not measure it. The field is locked in a self-reinforcing equilibrium: task formulations define single-shot functions, benchmarks evaluate them, leaderboards rank them, funding follows leaderboard position. Persistence is not merely unmeasured; it is unrepresentable in the current evaluation infrastructure. Breaking the equilibrium requires a benchmark that makes persistence the independent variable and future performance the dependent variable, per Section~\ref{sec:formalization}.
The longitudinal benchmark harness is defined by eight mandatory
properties.
\begin{enumerate}
  \item \textbf{Sequential presentation:} documents arrive in
    temporal order across explicit session boundaries, and the
    system's state must persist between sessions.
  \item \textbf{Cross-document dependency:} at least one question at
    session $t$ requires evidence that appeared only in sessions
    $t' < t$, so a system with no persistent state cannot answer
    correctly.
  \item \textbf{Evolving content:} the stream includes controlled
    revisions, amendments, and supersessions, requiring the system
    to update its state rather than simply accumulate.
  \item \textbf{Consolidation opportunities:} the stream includes
    repeated instances of a document family, rewarding systems that
    extract a reusable schema over those that re-process each
    instance from scratch.
  \item \textbf{Forgetting requirements:} at designated points, a
    deletion request targets specific evidence, and the system must
    demonstrate that the target is unrecoverable while unrelated
    knowledge remains intact.
  \item \textbf{Provenance queries:} the system must be able to
    trace any stored conclusion through its consolidation chain back
    to source bounding boxes.
  \item \textbf{Session boundaries:} the harness enforces explicit
    breaks between sessions so that a long context window cannot
    substitute for persistent state; the context is cleared at each
    boundary and only the persistent store $S_t$ survives.
  \item \textbf{Abstention:} at designated sessions the stream
    includes unanswerable queries, and the system must abstain
    rather than answer from stale or contaminated state.
\end{enumerate}
The eight properties correspond one-to-one to the audit criteria of
Table~\ref{tab:audit}: sequential presentation $\to$ Sequential;
cross-document dependency $\to$ Cross-doc; consolidation
opportunities $\to$ Family; evolving content $\to$ Temporal;
provenance queries $\to$ Provenance; forgetting requirements $\to$
Forgetting; session boundaries $\to$ Session; abstention $\to$
Answerability.

The decisive experiment is counterfactual and three-armed: at
evaluation point $k+1$, under identical future documents
$D_{k+1}$ and model weights $\theta$, we compare
(i)~a \emph{warm-state} condition carrying consolidated state
$S_k$;
(ii)~a \emph{retrieval-only} condition carrying an append-only
index $R_k$ over the same stream but performing no consolidation,
conflict resolution, or forgetting (the monotonic append-only
regime of Section~\ref{sec:architecture}); and
(iii)~a \emph{cold} condition from $S_0 = \emptyset$.
Define the \emph{access gain}
$\Delta_Q^{\mathrm{acc}} = Q(D_{k+1}\mid R_k) - Q(D_{k+1}\mid S_0)$
and the \emph{consolidation gain}
$\Delta_Q^{\mathrm{cons}} = Q(D_{k+1}\mid S_k) - Q(D_{k+1}\mid R_k)$.
Experience Gain is the consolidation gain,
$EG(k) \triangleq \Delta_Q^{\mathrm{cons}}(k)$;
the access gain is reported as a control.
\begin{equation}
\Delta_Q = Q(D_{k+1} \mid S_k) - Q(D_{k+1} \mid S_0)
\label{eq:deltaq}
\end{equation}
remains the total warm--cold contrast, decomposed as
$\Delta_Q = \Delta_Q^{\mathrm{acc}} + \Delta_Q^{\mathrm{cons}}$.
This isolates what consolidated state adds beyond retrieval
access. A positive consolidation gain does not by itself prove
memory: the harness also controls for contamination, duplicated
evidence, and uses stream-only information wherever possible to
minimize overlap with model pretraining.

\subsection{Counterfactual Evaluation Metrics}
\label{sec:metrics}
The counterfactual design yields five primary metrics. \emph{Experience Gain} is the consolidation gain of
Section~\ref{sec:harness}: $EG(k) \triangleq
\Delta_Q^{\mathrm{cons}}(k)$; a system that genuinely learns from related documents should show positive experience gain for some task families, growing with $k$ before saturating. \emph{Cost Efficiency} measures compute reduction attributable to prior exposure:
\begin{equation}
CE(k) = 1 - \frac{C(D_{k+1} \mid S_k)}{C(D_{k+1} \mid S_0)}
\label{eq:ce}
\end{equation}
where $C$ denotes inference cost in tokens or FLOPs. A memory system can be valuable even when accuracy saturates if it reduces the compute required for recurring document families. \emph{Memory Harm}, defined in Section~\ref{sec:formalization}, measures performance degradation caused by contaminated or outdated state; \emph{Forgetting Fidelity} and \emph{Coverage Retention}, also defined there, measure whether deletion is genuine and preserves unrelated knowledge. Together these five metrics characterize the benefit, cost, risk, and governability of persistent document state; no single metric suffices — a system that maximizes Experience Gain but ignores Memory Harm will amplify its own errors, and one with perfect Forgetting Fidelity but destroyed Coverage Retention cannot be trusted to remember selectively. The five metrics must be reported jointly.

\subsection{Statistical Design and Contamination Control}
\label{sec:statistical}

Document streams exhibit family-level dependence and serial
correlation that violate the independence assumptions of standard
statistical tests. We recommend a mixed-effects model that
directly estimates the warm--cold quality contrast underlying
Experience Gain (Equation~\eqref{eq:deltaq}):
\begin{multline}
Q_{ijkt} = \beta_0 + \beta_1 f(k) + \beta_2 W_t
+ \beta_3 f(k)\cdot W_t \\
+ u_j + b_j f(k) + s_i + \epsilon_{ijkt}
\label{eq:mixed}
\end{multline}
where $i$ indexes the system, $j$ the document family, $k$ the
exposure level, and $t$ the session within a stream. The condition
indicator $W_t \in \{0, 1\}$ distinguishes the warm path
($W_t = 1$, system carries state $S_k$) from the cold baseline
($W_t = 0$, system begins from $S_0 = \emptyset$). The saturating
exposure function $f(k) = \log(1 + k)$ ensures that the estimated
experience effect grows with $k$ and asymptotically saturates,
consistent with the paper's claim that Experience Gain increases
before reaching a ceiling. Because Equation~\eqref{eq:mixed} contains both the main effect
$\beta_2 W_t$ and the interaction $\beta_3 f(k)\cdot W_t$, the
warm--cold contrast at exposure $k$ is $\beta_2 + \beta_3 f(k)$.
The estimated Experience Gain is therefore
$\widehat{EG}(k) = \hat{\beta}_2 + \hat{\beta}_3 f(k)$;
the interaction coefficient $\beta_3$ captures the marginal gain
per unit of log-exposure, while $\beta_2$ captures the fixed
advantage of carrying any prior state.

\begin{figure*}[!t]
\centering
\begin{tikzpicture}[font=\scriptsize,
  box/.style={rectangle, draw, rounded corners=1.5pt, minimum height=0.6cm, align=center},
  arr/.style={-{Stealth[length=1.8mm]}, thick},
  rarr/.style={-{Stealth[length=1.8mm]}, thick, red!70!black},
  ptitle/.style={font=\small\bfseries}]

\begin{scope}[shift={(0,0)}]
  \node[ptitle] at (2.5,3.0) {(a) Memory contamination};
  \node[box, minimum width=1.9cm, fill=gray!10] (obs) at (2.5,2.1)
    {Incoming $o_t$\\error w.p.\ $p_e$};
  \node[box, minimum width=1.6cm, fill=blue!8] (val) at (2.5,0.9) {Validate};
  \node[box, minimum width=1.5cm, fill=green!10] (clean) at (0.9,-0.5) {Clean\\write};
  \node[box, minimum width=1.5cm, fill=red!14] (contam) at (3.9,-0.5) {Contam.\\$S_t$};
  \draw[arr] (obs) -- (val);
  \draw[arr] (val) -- (clean);
  \draw[rarr] (val) -- (contam);
  \node[font=\tiny] at (1.5,0.35) {pass};
  \node[font=\tiny, red!70!black] at (3.4,0.35) {fail $p_c$};
  \draw[rarr, dashed] (contam.east) -- ++(0.5,0) |- (obs.east);
  \node[font=\scriptsize] at (2.5,-1.7)
    {$P_{\mathrm{cont}}(k)=1-(1-p_e p_c)^k \to 1$};
\end{scope}

\begin{scope}[shift={(6.2,0)}]
  \node[ptitle] at (2.4,3.0) {(b) Bitemporal supersession};
  \draw[arr] (0.2,-1.4) -- (4.9,-1.4) node[right, font=\tiny]{valid time};
  \fill[blue!22] (0.4,0.9) rectangle (2.9,0.5);
  \draw (0.4,0.9) rectangle (2.9,0.5);
  \node[font=\tiny] at (1.65,0.7) {$D_1$ active};
  \node[font=\tiny] at (1.65,0.2) {valid $[t_a,t_c)$};
  \fill[gray!18] (2.9,0.9) rectangle (4.7,0.5);
  \draw[dashed, gray] (2.9,0.9) rectangle (4.7,0.5);
  \node[font=\tiny, gray] at (3.8,0.7) {superseded};
  \fill[orange!28] (2.9,-0.1) rectangle (4.7,-0.5);
  \draw (2.9,-0.1) rectangle (4.7,-0.5);
  \node[font=\tiny] at (3.8,-0.3) {$D_2$ active};
  \node[font=\tiny] at (3.8,-0.8) {valid $[t_c,t_e)$};
  \draw[dashed, red!60!black] (2.9,1.2) -- (2.9,-1.1);
  \node[font=\tiny, red!60!black] at (2.9,1.35) {$t_c$};
  \node[font=\tiny, align=center] at (1.65,-0.4) {$D_1$ retained\\for audit};
  \node[font=\tiny] at (3.9,1.7) {ingest $\tau_2>\tau_1$};
\end{scope}

\begin{scope}[shift={(12.4,0)}]
  \node[ptitle] at (2.3,3.0) {(c) Unlearning hierarchy};
  \node[box, minimum width=3.4cm, fill=gray!8, text=black] (l1) at (2.3,-1.6) {L1: Physical deletion};
  \node[box, minimum width=3.4cm, fill=gray!8, text=black] (l2) at (2.3,-0.75) {L2: Representation removal};
  \node[box, minimum width=3.4cm, fill=red!6, text=black, dashed, draw=red!70!black, line width=0.8pt] (l3) at (2.3,0.1) {L3: Schema repair};
  \node[box, minimum width=3.4cm, fill=red!6, text=black, dashed, draw=red!70!black, line width=0.8pt] (l4) at (2.3,0.95) {L4: Parametric unlearning};
  \node[box, minimum width=3.4cm, fill=red!6, text=black, dashed, draw=red!70!black, line width=0.8pt] (l5) at (2.3,1.8) {L5: Certified non-recover.};
  \draw[arr] (l1.north) -- (l2.south);
  \draw[arr] (l2.north) -- (l3.south);
  \draw[arr] (l3.north) -- (l4.south);
  \draw[arr] (l4.north) -- (l5.south);
  \draw[thick, red!70!black] (4.2,-0.3) -- (4.2,2.2);
  \node[font=\tiny, red!70!black, rotate=90, anchor=north] at (4.45,0.95) {the gap};
\end{scope}

\end{tikzpicture}
\caption{The three governance challenges of persistent
evidence-grounded state. (a) Memory contamination through
unvalidated writes. (b) Bitemporal supersession of evolving
documents. (c) Five-level unlearning hierarchy; current Document
AI implements Levels 1--2 only.}
\label{fig:governance}
\end{figure*}
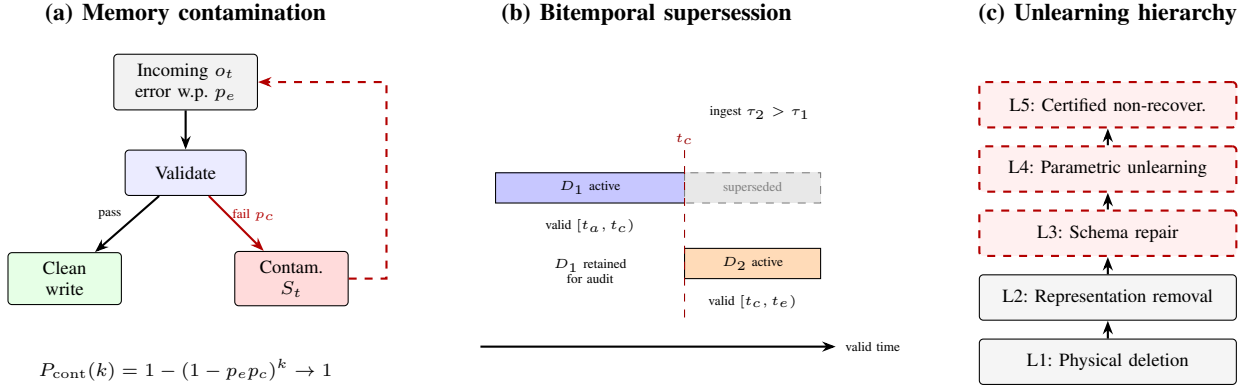

The random effects account for structured heterogeneity:
$u_j \sim \mathcal{N}(0, \sigma^2_{\mathrm{fam}})$ is a
family-level random intercept capturing baseline difficulty
differences across document families, and
$b_j \sim \mathcal{N}(0, \sigma^2_{\mathrm{slope}})$ is a
family-level random slope allowing different families to benefit
from exposure at different rates. System is treated as a fixed
effect $s_i$ because the harness compares a specific, finite set
of systems rather than a random sample from a population of
possible architectures. Residuals within a stream exhibit serial
correlation; we model this with a first-order autoregressive
structure
$\mathrm{Corr}(\epsilon_{t}, \epsilon_{t+1}) = \rho$
estimated jointly with the fixed effects. For binary task outcomes,
a logistic mixed model with the same structure is appropriate.
Paired bootstrap confidence intervals over independently sampled
streams provide a distribution-free alternative when the number
of streams is small.

\paragraph{Controlled corruption for Memory Harm.} Evaluating Memory Harm requires a contamination-free baseline $S_k^{\mathrm{clean}}$, built by running two parallel streams over the same document sequence: a \emph{clean stream} verifying all observations before consolidation, and a \emph{corrupted stream} replacing a controlled fraction $\rho_c$ of observations at designated sessions with plausible but incorrect values (e.g., a misread tax ID, a transposed table cell), both under identical weights and validation settings. The estimate is the quality difference at the evaluation point: $\widehat{MH}(k) = Q(D_{k+1} \mid S_k^{\mathrm{clean}}) - Q(D_{k+1} \mid S_k^{\mathrm{corrupted}})$. This construction ties the corruption injection mechanism of Section~\ref{sec:contamination} directly to Definition~\ref{def:mh}. To ensure measured Memory Harm reflects stream-induced state rather than pretraining knowledge, the harness minimizes overlap between benchmark content and pretraining corpora, records all retrieval access, and prefers synthetic document families generated for the benchmark.

\paragraph{Cumulative Cost Efficiency.}
The Cost Efficiency metric in Equation~\eqref{eq:ce} evaluates
a single future task. The economically relevant quantity, however,
is cumulative stream cost: the total compute expended over the
entire stream $D_1, \dots, D_{k+1}$ under warm versus cold
conditions. We therefore define the cumulative cost ratio:
\begin{equation}
CE_{\mathrm{cum}}(k) = 1 -
\frac{\sum_{t=1}^{k+1} C(D_t \mid S_{t-1}^{\mathrm{warm}})}
{\sum_{t=1}^{k+1} C(D_t \mid S_{t-1}^{\mathrm{cold}})}
\label{eq:ce_cum}
\end{equation}
where $C$ includes perception, retrieval, reasoning, and memory
maintenance costs. A memory system is economically justified when
$CE_{\mathrm{cum}}(k) > 0$ for the target stream length, even if
per-task Cost Efficiency at any individual session is zero.

This completes the evaluation instrument. The audit in Table~\ref{tab:audit} makes the gap visible. The harness makes the gap measurable. The five metrics make the gap falsifiable: a system that achieves positive Experience Gain, positive Cost
Efficiency, bounded Memory Harm, and high Forgetting Fidelity
with Coverage Retention under the eight mandatory properties has
demonstrated persistent evidence-grounded document state under our operational definition. A system that fails any of these tests has not. The specification does not privilege an
architecture. It privileges a behavior. The next section addresses what happens when that behavior goes wrong, when memory becomes a liability rather than an asset, and what governance is required to prevent it.

\section{Governance, Failure Modes, and Research Agenda}
\label{sec:governance}

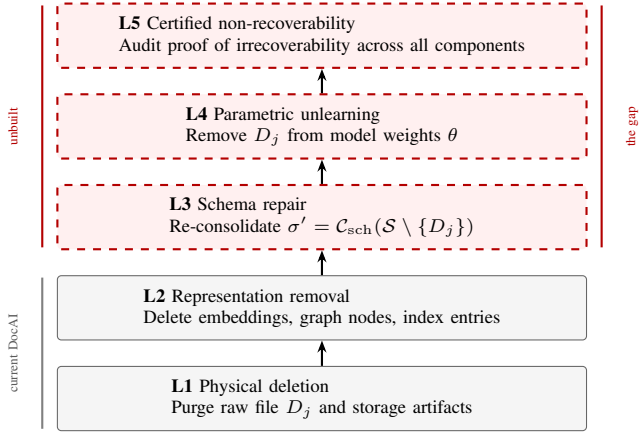
\begin{figure}[!t]
\centering
\begin{tikzpicture}[font=\scriptsize,
  lvl/.style={rectangle, draw=black!60, rounded corners=1.5pt, minimum width=7.0cm, minimum height=0.85cm, align=left, fill=gray!8, text=black},
  gap/.style={rectangle, draw=red!70!black, rounded corners=1.5pt, minimum width=7.0cm, minimum height=0.85cm, align=left, fill=red!6, text=black, dashed, line width=0.8pt},
  arr/.style={-{Stealth[length=1.6mm]}, thick, black},
  lbl/.style={font=\tiny, gray!70!black}]

\node[gap] (l5) at (0,5.1) {\textbf{L5} Certified non-recoverability\\
{\scriptsize Audit proof of irrecoverability across all components}};
\node[gap] (l4) at (0,3.9) {\textbf{L4} Parametric unlearning\\
{\scriptsize Remove $D_j$ from model weights $\theta$}};
\node[gap] (l3) at (0,2.7) {\textbf{L3} Schema repair\\
{\scriptsize Re-consolidate $\sigma' = \mathcal{C}_{\mathrm{sch}}(\mathcal{S}\setminus\{D_j\})$}};
\node[lvl] (l2) at (0,1.5) {\textbf{L2} Representation removal\\
{\scriptsize Delete embeddings, graph nodes, index entries}};
\node[lvl] (l1) at (0,0.3) {\textbf{L1} Physical deletion\\
{\scriptsize Purge raw file $D_j$ and storage artifacts}};

\draw[arr] (l1.north) -- (l2.south);
\draw[arr] (l2.north) -- (l3.south);
\draw[arr] (l3.north) -- (l4.south);
\draw[arr] (l4.north) -- (l5.south);

\draw[thick, gray] (-3.7,-0.1) -- (-3.7,1.9);
\node[lbl, rotate=90, anchor=south] at (-3.9,0.9) {current DocAI};

\draw[thick, red!70!black] (-3.7,2.3) -- (-3.7,5.5);
\node[font=\tiny, red!70!black, rotate=90, anchor=south] at (-3.9,3.9) {unbuilt};

\draw[thick, red!70!black] (3.7,2.3) -- (3.7,5.5);
\node[font=\tiny, red!70!black, rotate=90, anchor=north] at (3.9,3.9) {the gap};

\end{tikzpicture}
\caption{Five-level hierarchy for document-aware unlearning.
Passing one level does not imply success at the next. Current
Document AI implements Levels 1--2; Levels 3--5 remain unbuilt.
Levels 1--3 are evaluated by Forgetting Fidelity ($FF$) and
Coverage Retention ($CR$); Levels 4--5 require additional
certification.}
\label{fig:unlearning}
\end{figure}

\subsection{Memory Contamination and the Experience-Cost Curve}
\label{sec:contamination}
Remembering introduces failure modes that stateless pipelines do not exhibit: a persistent system can propagate a single error across thousands of future sessions, making the error self-reinforcing. The Memory Harm metric and the contamination-accumulation bound of Theorem~\ref{thm:mh} quantify this risk; evaluating it requires benchmarks that inject controlled corruption into the document stream. The design challenge is to build memory that is robust to its own fallibility. Figure~\ref{fig:governance} summarizes the three governance challenges addressed in this section.

When memory is robust, its primary value is economic. A Document AI system that processes every recurring invoice from pixels at full resolution, re-running layout detection and structure parsing on every instance, has failed to exploit repetition; persistent state avoids this redundant work. We therefore propose an experience--cost curve decomposing inference cost into perception, retrieval, reasoning, and memory maintenance under two conditions: the \emph{warm} condition, in which the system has processed a stream $D_1, \dots, D_k$ and retains persistent state $S_k$, and the \emph{cold} condition, in which the system begins from $S_0 = \emptyset$ and does not process the stream (Figure~\ref{fig:harness}). A memory-enabled system should progressively shift from instance-level perception toward schema-conditioned processing, where the consolidated template acts as a top-down prior accelerating layout parsing and focusing attention on deviations. The relevant comparison is not per-document latency but cumulative stream compute: total cost warm versus cold. A memory system can be valuable even with non-zero per-document overhead, provided cumulative savings from schema reuse exceed the cost of maintaining state.

\subsection{Recursive Provenance and Bitemporal Supersession}
\label{sec:provenance}
Schema-conditioned processing introduces a tension between compression and auditability: a schema abstracts away the details that regulated domains must recover when a conclusion is challenged. Provenance must therefore survive abstraction, preserving the recursive chain $\sigma_2 \rightsquigarrow \sigma_1 \rightsquigarrow \{e_1, e_2, e_3\} \rightsquigarrow \{D_i, \ell_i, v_i\}$~\cite{buneman2001why,cheney2009provenance,moreau2013prov}. This recursive provenance is the precondition for trust: a system that cannot point to the exact table cell and footnote that justify a consolidated claim cannot be deployed in an audit.

\begin{algorithm}[t]
\caption{Five-Level Verifiable Forgetting}
\label{alg:forget}
\begin{algorithmic}[1]
\Require State $S_t$; deletion request $X_{\mathrm{del}}$
\Ensure Repaired state $S_t'$; audit certificate bounding residual leakage
\State \textbf{L1 Physical:} $\textsc{DeleteRaw}(X_{\mathrm{del}})$
    \Comment{PDFs, page images, OCR cache}
\State \textbf{L2 Representation:} $\textsc{RemoveEmb}(X_{\mathrm{del}})$;
    $\textsc{RemoveNodes}(\Gamma_t,X_{\mathrm{del}})$
\State \textbf{L3 Schema repair:}
\For{each schema $\sigma \in K_t$ with $I(X_{\mathrm{del}};\sigma)>0$}
\State $\sigma'\leftarrow\mathcal{C}_{\mathrm{sch}}(E_T \setminus E_{X_{\mathrm{del}}},\, K_{T-1},\, \Gamma_{T-1})$
    \State replace $\sigma$ by $\sigma'$ in $K_t$; repair $P_t$ links
\EndFor
\State \textbf{L4 Parametric:} $\textsc{PruneWeights}(\theta,X_{\mathrm{del}})$
\State \textbf{L5 Certify:} $\textsc{AuditLeakage}(X_{\mathrm{del}}, S_t')$
    \Comment{probe-based upper bound on $\operatorname{Leak}$}    
\State \Return $S_t'$, certificate
\end{algorithmic}
\end{algorithm}

\subsection{Verifiable Forgetting and the Unlearning Hierarchy}
\label{sec:unlearning}
Trust also requires the ability to revoke it. When a data subject
exercises the right to erasure under Article~17 of the
GDPR~\cite{regulation2016regulation}, or when protected health information is subject
to disposal under the privacy and security safeguards of
HIPAA~\cite{annas2003hipaa}, the system must execute verifiable forgetting across the entire state. We formalize this as a five-level  unlearning hierarchy (Figure~\ref{fig:unlearning}). Level~1 is physical object
deletion, purging the raw document file. Level~2 is representation
removal, invalidating the corresponding embeddings or graph nodes.
Level~3 is consolidated schema repair, re-abstracting the semantic
store to remove any information that depended on the deleted evidence.
Level~4 is parametric unlearning~\cite{cao2015towards,bourtoule2021machine,nguyen2025survey}, pruning the incorporated facts from model weights. The
broader machine-unlearning literature provides exact unlearning via
retraining~\cite{kang2024machine}, approximate unlearning via
influence-based updates~\cite{he2026trace}, and certified unlearning
via residual checks~\cite{sehgal2026forgetting}; adapting these to
the document-native schema-repair setting of Level~3 is an open
algorithmic problem. Level~5 is certified non-recoverability, generating a
cryptographic or statistical audit proving that the deleted information
cannot be reconstructed. Algorithm~\ref{alg:forget} specifies this
procedure: each level addresses a distinct persistence mechanism, and
the certificate at Level~5 attests that Forgetting Fidelity has been
satisfied. Current Document AI systems operate almost exclusively at
Levels~1 and~2. The formalization in
Section~\ref{sec:formalization} showed that skipping Level~3 leaves
non-zero information leakage in the semantic store. Building Level~3
repair is therefore an engineering and algorithmic imperative. 

\subsection{Enterprise Deployment Constraints}
\label{sec:deployment}
Robust memory, cumulative efficiency, recursive provenance, and verifiable unlearning map directly to operational bottlenecks in the industries that stand to gain most from Document AI. Legal discovery review teams process millions of documents over months, where the inability to accumulate findings across sessions multiplies cost while inability to track non-monotonic contract amendments adds unacceptable risk~\cite{cormack2014evaluation}. Healthcare relies on longitudinal patient models to detect contradictions between successive laboratory reports and medication histories~\cite{rajkomar2018scalable}; a system that resets at every encounter forces manual reconciliation. Financial auditing requires an unbroken provenance chain from a consolidated financial statement back to individual transaction receipts~\cite{hasan2022artificial}; a system that abstracts without retaining the chain fails the audit. Scientific literature synthesis requires tracking how consensus evolves, identifying when a new paper supersedes an older finding rather than merely adding to it~\cite{marshall2016robotreviewer}. In every domain, Document AI's value proposition depends on accumulation, yet the current literature provides only episodic competence.

Bridging this gap requires systems engineered for deployment constraints. Enterprise document streams arrive at high throughput, so persistent memory must remain efficient as accumulated information grows; consolidation is expensive and need not occur in real time, creating a trade-off between consistency, availability, and cost~\cite{vogels2009eventually} --- some applications may tolerate eventually consistent memory while others require updates to become effective immediately. Multi-tenant isolation is critical: a law firm's document memory must enforce strict boundaries between matters, clients, and attorneys. Governance must extend beyond tenant isolation because different parts of the same document carry different permissions --- clauses may be visible to a department while pricing or negotiated terms remain restricted to finance or leadership --- and policy-aware governed memory must preserve these restrictions in what the system retains and uses~\cite{hu2013guide}. Persistence also creates security risks: malicious or corrupted information entering memory may influence many later sessions rather than a single interaction~\cite{gao2023retrieval}. Retention and data-residency requirements must extend from source documents to the information remembered from them~\cite{regulation2016regulation,annas2003hipaa}, while maintaining state must remain justified by cumulative benefit. A memory architecture that cannot keep what it remembers consistent, secure, and governed across users and time cannot leave the laboratory.

\subsection{An Implementation-Neutral Research Agenda}
\label{sec:agenda}
These deployment realities define an implementation-neutral research
agenda. The field must address six open problems. First,
provenance-preserving consolidation: how to abstract across
heterogeneous visual observations into compact schemas without
destroying the evidentiary chain to source bounding boxes. Second,
adaptive evidence acquisition: formulating optimal stopping policies
where an agentic system halts visual page inspection when the expected
information gain falls below the compute cost. Third,
memory-conditioned top-down parsing: demonstrating empirically that
warm-state parsers leverage learned spatial schemas to accelerate
layout detection on recurring document families. Fourth, bitemporal
state and non-monotonic supersession: modeling the distinct axes of
publication time, event time, effective time, and ingestion time to
resolve document versioning conflicts without discarding historical
audit trails. Fifth, certified machine unlearning: algorithms for
repairing consolidated knowledge graphs and parametric priors upon
receiving data erasure requests. Sixth, policy-aware governed memory:
integrating multi-tenant access control policies directly into
late-interaction retrieval and the state representation itself. These
questions are deliberately decoupled from any specific substrate. A
graph database, a differentiable memory module, a hybrid
vector-symbolic architecture, or a future substrate not yet invented
can be evaluated against them.

\subsection{Limitations}
\label{sec:limitations}
Although the proposed formalization is structurally agnostic to
language, script, and document era, our literature survey is
scoped to the dominant published corpora, which are biased toward
English-language, Latin-script, modern documents. Document AI for
low-resource languages and historical manuscripts is therefore
underrepresented in our review, though the framework itself is
directly applicable to them. The cognitive science literature is
used here to generate computational hypotheses, not to establish
that biological mechanisms are optimal engineering solutions; the
mapping between cortical consolidation and schema formation is an
architectural analogy, not a biological claim. The proposed state
tuple, lifecycle operations, and benchmark harness are formal
specifications, not validated implementations; we have not
empirically demonstrated that a system satisfying all these
properties can be built efficiently at scale. The strongest
methodological limitation of the proposed harness is the
difficulty of obtaining genuinely novel document streams that are
absent from a base model's training data. A credible longitudinal
benchmark must report its contamination controls and should
include private or synthetically generated document families. The
underlying principles: validated persistence, provenance-linked
consolidation, temporal supersession, and governed forgetting, are
domain-agnostic and may inform persistent-state design in adjacent
fields such as clinical decision support, legal reasoning, and
scientific knowledge management.

\section{Conclusion}
\label{sec:conclusion}

This survey began with a paradox: the most capable document
systems ever built share a property with the simplest ones.
They have no memory. We traced the field's evolution from
handcrafted pipelines to layout-aware pretraining, from
OCR-free generation to visual retrieval, from single-page
question answering to long-document evidence acquisition, and
showed that at every stage the unit of competence expanded
while the unit of persistence remained fixed at a single
session. We named the resulting deficit the statelessness
bottleneck and demonstrated that it persists through
parameter scaling, context extension, and retrieval
augmentation, because these mechanisms address working-memory
capacity and information access but not knowledge
accumulation. Storage is not retrieval, and retrieval is not
consolidation.

We formalized persistent evidence-grounded document state as a
six-component tuple with lifecycle operations and invariants; we
derived a contamination-accumulation theorem showing that
unvalidated writes drive state corruption toward certainty over
long horizons, and that a validation gate bounds the accumulation
rate; and we specified five counterfactual metrics that make the
presence of persistent state falsifiable. This survey provides the
specification, the audit, and the evaluation protocol; the
implementations that satisfy them remain the field's open task.
These contributions are formal specifications, not validated
implementations: whether a system satisfying all six conjuncts can
be built efficiently at scale is an open empirical question.

The decisive experiment for the next era of Document AI is a
controlled counterfactual: the same base model, the same
future documents, and different valid prior histories,
resulting in different future competence. A system that
demonstrates positive Experience Gain, positive Cost
Efficiency, bounded Memory Harm, high Forgetting Fidelity,
and Coverage Retention under the eight mandatory properties
of the harness has crossed the boundary from passive access
to active consolidation. A system that fails any of these tests has
not. The field does not need a larger model to cross this
boundary. It needs a different architecture, a different task formulation, and a different evaluation protocol.

Document AI has industrialized reading. The systems built
over the past decade can perceive, parse, retrieve, and
reason over visually rich documents with a fluency that would
have seemed implausible a generation ago. But reading is an
episodic act, and the institutions that produce and consume
documents require systems that accumulate, reconcile, and
govern knowledge across time. The field has learned to read
documents. It must now learn to remember them. The vocabulary,
the formalism, and the evaluation protocol are in place. What
remains is the engineering, the experimentation, and the
benchmarks that will prove the transition. We hope this survey
provides the ground on which that work is built.

\bibliographystyle{IEEEtran}
\bibliography{custom}

\begin{IEEEbiography}
[{\includegraphics[width=1in,height=1.25in,clip,keepaspectratio]{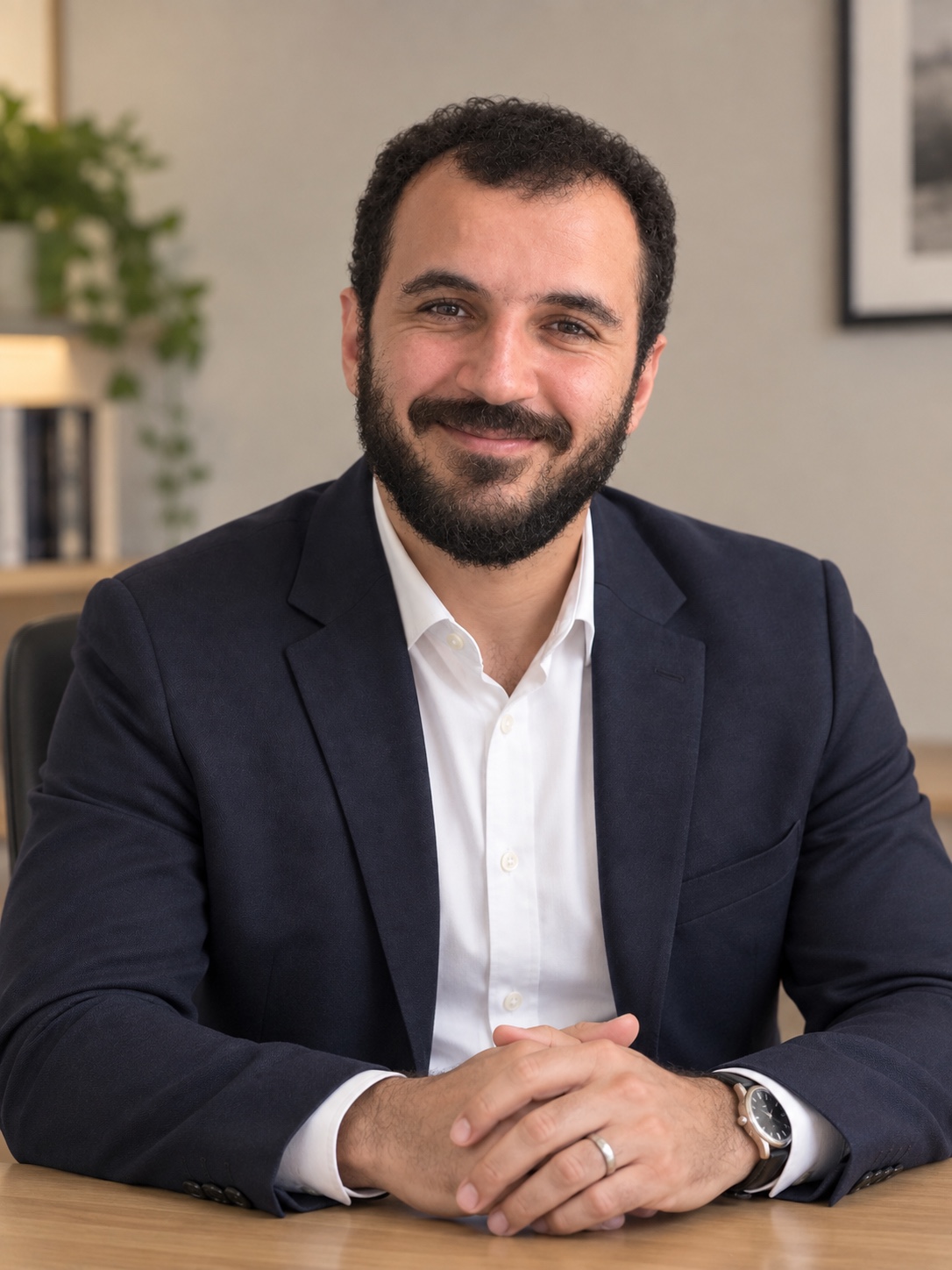}}]
{Souhail Bakkali} is an Associate Professor of Computer Science at Université de Rennes, France, and a researcher at the IRISA laboratory. He received his Ph.D. in Computer Science and Applications from La Rochelle Université, France, in 2022. His research interests include multimodal representation learning, vision-language models, and Document AI, with a particular focus on understanding and reasoning over complex visual and textual information. He supervises undergraduate and graduate students in natural language processing, computer vision, and multimodal learning. He serves as a reviewer for journals and conferences in pattern recognition, computer vision, and machine learning, and contributed to the organization of the Document Analysis Systems (DAS) workshop in 2023. He received an Outstanding Reviewer Award at ICDAR 2025.
\end{IEEEbiography}

\begin{IEEEbiography}
[{\includegraphics[width=1in,height=1.25in,clip,keepaspectratio]{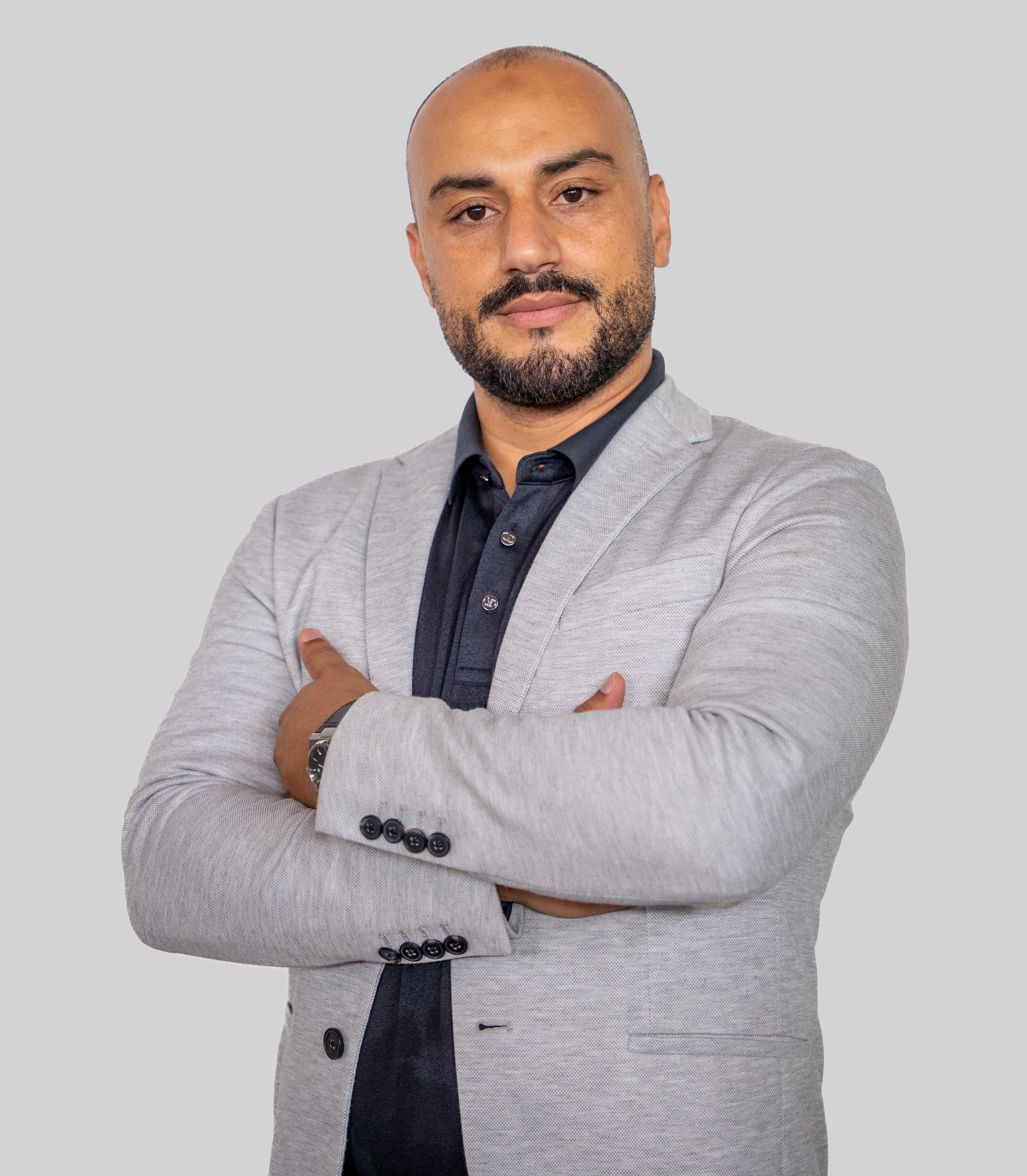}}]
{Ayoub Merimi}
is a Ph.D. student at Hassan First University, ENSA Berrechid, Morocco. His research lies at the intersection of Decision AI and Document AI, with a particular focus on persistent-state architectures, memory-enabled systems, and longitudinal reasoning over evolving information. He is also a software engineering leader and architect with more than 15 years of industry experience in large-scale information retrieval, distributed systems, personalization, and high-availability platforms. His career includes senior technical and leadership roles at Amazon, ABYAT, Amadeus, and Schlumberger, spanning global-scale software systems and AI-driven platforms. His research brings together Decision AI, persistent Document AI, and his experience with large-scale systems, with a focus on intelligent systems that can reason over evolving information across time and operate reliably at scale.
\end{IEEEbiography}

\end{document}